\documentclass{article}
\usepackage{graphicx}
\usepackage{iclr2024_conference}
\usepackage{times}
\iclrfinalcopy
\usepackage{amsmath}
\usepackage{amssymb}
\usepackage{mathtools}
\usepackage{amsthm}
\usepackage{color}
\usepackage{colortbl}

\definecolor{plum}  {rgb}{.4,0,.4}
\definecolor{brickred} {rgb}{0.6,0,0}
\usepackage[plainpages=false,pdfpagelabels,colorlinks=true,linkcolor=brickred,citecolor=brickred]{hyperref}

\usepackage[capitalize,noabbrev]{cleveref}
\input{packages}
\newcommand*{\diam}{{\mathsf{diam}}}

\newcommand*{\indc}{{\mathbf{1}}}
\newcommand*{\defeq}{\triangleq}

\newcommand{\relu}{\textsc{ReLU}}
\newcommand{\gelu}{\textsc{GELU}}
\newcommand{\silu}{\textsc{SiLU}}
\newcommand{\celu}{\textsc{CELU}}
\newcommand{\selu}{\textsc{SELU}}
\newcommand{\softplus}{\textsc{Softplus}}
\newcommand{\mish}{\textsc{Mish}}
\newcommand{\elu}{\textsc{ELU}}
\newcommand{\sigmoid}{\textsc{Sigmoid}}
\newcommand{\tanhh}{\textsc{Tanh}}

\newcommand{\relul}{\textsc{ReLU-Like}}

\DeclareFontEncoding{LS1}{}{}
\DeclareFontSubstitution{LS1}{stix}{m}{n}
\DeclareSymbolFont{stixletters}{LS1}{stix}{m}{it}
\DeclareMathAccent{\cev}{\mathord}{stixletters}{"91}
\DeclareMathAccent{\vec}{\mathord}{stixletters}{"92}
\DeclareMathAccent{\vecev}{\mathord}{stixletters}{"95}

\newcommand*{\bd}{\mathit{bd}}
\newcommand*{\interior}{\mathit{int}}

\newtheorem{theorem}{Theorem}

\newtheorem{claim}[theorem]{Claim}

\newtheorem{lemma}[theorem]{Lemma}

\theoremstyle{acmdefinition}
\newtheorem{definition}{Definition}

\title{Minimum width for universal approximation\\ using ReLU networks on compact domain}
\author{Namjun Kim$^1$\quad Chanho Min$^2$\quad Sejun Park$^1$\thanks{corresponding author\\$~$\quad~~ emails: namjun-kim@korea.ac.kr, chanhomin@ajou.ac.kr, sejun.park000@gmail.com}\\
$^1$Korea University\quad$^2$Ajou University}
\date{}

\begin{document}
\maketitle

\vspace{-0.2in}
\begin{abstract}
\noindent
It has been shown that deep neural networks of a large enough width are universal approximators but they are not if the width is too small.
There were several attempts to characterize the minimum width $w_{\min}$ enabling the universal approximation property; however, only a few of them found the exact values.
In this work, we show that the minimum width for %
$L^p$ approximation of $L^p$ functions from $[0,1]^{d_x}$ to $\mathbb R^{d_y}$ is exactly $\max\{d_x,d_y,2\}$ if an activation function is $\relul$ (e.g., $\relu$, $\gelu$,  $\softplus$).
Compared to the known result for $\relu$ networks, $w_{\min}=\max\{d_x+1,d_y\}$ when the domain is $\smash{\mathbb R^{d_x}}$, our result first shows that approximation on a compact domain requires smaller width than on $\smash{\mathbb R^{d_x}}$.
We next prove a lower bound on $w_{\min}$ for uniform approximation using general activation functions including $\relu$: $w_{\min}\ge d_y+1$ if $d_x<d_y\le2d_x$. Together with our first result, this shows a dichotomy between $L^p$ and uniform approximations for general activation functions and input/output dimensions.  
\end{abstract}

\vspace{-0.1in}
\section{Introduction}\label{sec:intro}
\vspace{-0.05in}
Understanding what neural networks can or cannot do is a fundamental problem in the expressive power of neural networks. Initial approaches for this problem mostly focus on depth-bounded networks. For example, a line of research studies the size of the two-layer neural network to memorize (i.e., perfectly fit) an arbitrary training dataset and shows that the number of parameters proportional to the dataset size is necessary and sufficient for various activation functions \citep{baum88,huang98}.
Another important line of work investigates a class of functions that two-layer networks can approximate. %
Classical results in this field represented by the universal approximation theorem show that two-layer networks using a non-polynomial activation function are dense in the space of continuous functions on compact domains \citep{hornik89, cybenko89, journals/nn/Leshno93, pinkus99}.

With the success of deep learning, the expressive power of deep neural networks has been studied. 
As in the classical depth-bounded network results, several works have shown that width-bounded networks can memorize arbitrary training dataset \citep{yun19,vershynin20} and can approximate any continuous function \citep{Lu17,hanin17}.
Intriguingly, it has also been shown that deeper networks can be more expressive compared to shallow ones.
For example, 
\citet{telgarsky16,eldan16,daniely17}
show that there is a class of functions that can be approximated by deep width-bounded networks with a small number of parameters but cannot be approximated by shallow networks without extremely large widths.
Furthermore, width-bounded networks require a smaller number of parameters for universal approximation \citep{yarotsky18} and memorization \citep{park21b,vardi22} compared to depth-bounded ones. 

Recently, researchers started to identify the \emph{minimum width} that enables universal approximation of width-bounded networks as a dual problem of the classical results: the minimum depth of neural networks for universal approximation is \emph{exactly two} if their activation function is non-polynomial. 
Unlike the minimum depth independent of the input dimension $d_x$ and the output dimension $d_y$ of target functions, the minimum width is known to lie between $d_x$ and $d_x+d_y+\alpha$ for various activation functions where $\alpha$ is some non-negative number depending on the activation function \citep{Lu17,hanin17,johnson18,kidger20,park21,cai23}.
However, most existing results only provide bounds on the minimum width, and the exact minimum width is known for a few activation functions and problem setups so far. %

\begin{table}[t]
\begin{center}
\caption{%
A summary of known bounds on the minimum width for universal approximation. In this table, $p\in[1,\infty)$ and all results with the domain $[0,1]^{d_x}$ extends to an arbitrary compact set in $\mathbb R^{d_x}$.}
\label{table:summary}
\begin{tabular}{| c | c  c | c |} 
 \hline
 {\bf Reference} & {\bf Function class} & {\bf Activation} $\sigma$ & {\bf Upper\,/\,lower bounds} \\ 
 \hline\hline
 \multirow{2}{*}{\cite{Lu17}} & $L^1(\mathbb R^{d_x}, \mathbb R)$ & $\relu$ & $d_x + 1 \le w_{\min} \le d_x + 4$ \\
                              & $L^1([0,1]^{d_x}, \mathbb R)$ & $\relu$ &  $w_{\min} \ge d_x$ \\
 \hline
 \cite{hanin17} & $C([0,1]^{d_x}, \mathbb R^{d_y})$ & $\relu$ & $d_x + 1 \le w_{\min} \le d_x + d_y$ \\
 \hline
 \cite{johnson18} & $C([0,1]^{d_x}, \mathbb R^{d_y})$ & uniformly conti.$^\dagger$ & $w_{\min} \ge d_x + 1$ \\
 \hline
 \multirow{3}{*}{\cite{kidger20}} & $C([0,1]^{d_x}, \mathbb R^{d_y})$ & conti. nonpoly.$^\ddagger$ & $w_{\min} \le d_x + d_y + 1$ \\
                                  & $C([0,1]^{d_x}, \mathbb R^{d_y})$ & nonaffine poly. &  $w_{\min} \le d_x + d_y + 2$ \\
                                  & $L^p(\mathbb R^{d_x}, \mathbb R^{d_y})$ & $\relu$ &  $w_{\min} \le d_x + d_y + 1$ \\
 \hline
 \multirow{3}{*}{\cite{park21}} & $L^p(\mathbb R^{d_x}, \mathbb R^{d_y})$ & $\relu$ &  $w_{\min} = \max\{d_x + 1, d_y\}$ \\ 
                                & $C([0,1], \mathbb R^2)$ & $\relu$ & $w_{\min} > \max\{d_x + 1, d_y\}$ \\
                                & $L^p([0,1]^{d_x},\mathbb R^{d_y})$ & conti. nonpoly.$^\ddagger$ & $w_{\min} \le \max\{ d_x + 2, d_y+1 \}$\\
 \hline
 \multirow{3}{*}{\cite{cai23}}& $L^p([0,1]^{d_x}, \mathbb R^{d_y})$ & Leaky-$\relu$ &  $w_{\min} = \max\{d_x, d_y, 2\}$ \\ 
 &$L^p([0,1]^{d_x}, \mathbb R^{d_y})$ & arbitrary &  $w_{\min} \ge \max\{d_x, d_y\}$ \\
 &$C([0,1]^{d_x}, \mathbb R^{d_y})$ & arbitrary &  $w_{\min} \ge \max\{d_x, d_y\}$ \\ 
  \hline
 \hline
 \rowcolor{gray!30} {\bf Ours} (\cref{thm:lp-ub}) & $L^p([0,1]^{d_x}, \mathbb R^{d_y})$ & $\relu$ & $ w_{\min} = \max\{d_x, d_y, 2\}$ \\
 \rowcolor{gray!30} {\bf Ours} (\cref{cor:general-lp}) & $L^p([0,1]^{d_x}, \mathbb R^{d_y})$ &$\relul^\mathsection$ & $w_{\min} = \max\{d_x, d_y, 2\}$ \\
 \rowcolor{gray!30} {\bf Ours} (\cref{thm:unif-lb-leakyrelu}) & $C([0,1]^{d_x}, \mathbb R^{d_y})$ %
 & {conti.$^\dagger$}
 & $ w_{\min} \ge d_y + \indc_{d_x < d_y \le 2 d_x}$ \\
 \hline
\end{tabular}
\end{center}
$\dagger$ requires that $\sigma$ is uniformly approximated by a sequence of continuous one-to-one functions.\\
$\ddagger$ requires that $\sigma$ is continuously differentiable at least one point $z$, with $\sigma'(z) \neq 0$.\\
$\mathsection$ includes $\softplus$, Leaky-$\relu$, $\elu$, $\celu$, $\selu$, $\gelu$, $\silu$, and $\mish$ where $\gelu$, $\silu$, and $\mish$ require $d_x+d_y\ge3$.\\
\vspace{-0.1in}
\end{table}

\subsection{Related works}
Before summarizing prior works, we first define function spaces often considered in universal approximation literature. We use $C(\mathcal X,\mathcal Y)$ to denote the space of all continuous functions from $\mathcal X\subset\mathbb R^{d_x}$ to $\mathcal Y\subset\mathbb R^{d_y}$, endowed with the uniform norm: ${\|f\|_\infty\defeq\sup_{x\in\mathcal X}\|f(x)\|_\infty}$.
We also define $L^p(\mathcal X,\mathcal Y)$ for denoting the $L^p$ space, i.e., the class of all functions with finite $L^p$-norm, endowed with the $L^p$-norm: ${\|f\|_p\defeq(\int_\mathcal X\|f\|_p^pd\mu_{d_x})^{1/p}}$ where $\mu_{d_x}$ denotes the $d_x$-dimensional Lebesgue measure. 
We denote the minimum width for universal approximation by $w_{\min}$.
See \cref{sec:setup} for more detailed problem setup.
Under these notations, \cref{table:summary} summarizes the known upper and lower bounds on the minimum width for universal approximation under various problem setups.

{\bf Initial approaches.} \citet{Lu17} provide the first upper bound $w_{\min}\le d_x+4$ for universal approximation of $L^1(\mathbb R^{d_x},\mathbb R)$ using (fully-connected) $\relu$ networks. They explicitly construct a network of width $d_x+4$ which approximates a target $L^1$ function by using $d_x$ neurons to store the $d_x$-dimensional input, one neuron to transfer intermediate constructions of the one-dimensional output, and the remaining three neurons to compute iterative updates of the output.
For multi-dimensional output cases, similar constructions storing the $d_x$-dimensional input and $d_y$-dimensional (intermediate) outputs
are used to prove upper bounds on $w_{\min}$ under various problem setups.
For example, \citet{hanin17} show that $\relu$ networks of width $d_x+d_y$ are dense in $C([0,1]^{d_x},\mathbb R^{d_y})$.
\citet{kidger20} also prove upper bounds on the minimum width for general activation functions
using similar constructions.
They prove $w_{\min}\le d_x+d_y+1$ for $C([0,1]^{d_x},\mathbb R^{d_y})$ if an activation function $\sigma$ is non-polynomial and $\sigma'(z)\ne0$ for some $z$, $w_{\min}\le d_x+d_y+2$ for $C([0,1]^{d_x},\mathbb R^{d_y})$ if $\sigma$ is non-affine polynomial, and $w_{\min}\le d_x+d_y+1$ for  $L^p(\mathbb R^{d_x},\mathbb R^{d_y})$ if $\sigma=\relu$. 

Lower bounds on the minimum width have also been studied. For $\relu$ networks, \citet{Lu17} show that the minimum width is at least $d_x+1$ and $d_x$ to universally approximate $L^p(\mathbb R^{d_x},\mathbb R^{d_y})$ and $L^p([0,1]^{d_x},\mathbb R^{d_y})$, respectively. \citet{johnson18} considers general activation functions and shows that the minimum width is at least $d_x+1$ to universally approximate $C([0,1]^{d_x},\mathbb R^{d_y})$ if an activation function is uniformly continuous and can be uniformly approximated by a sequence of continuous and one-to-one functions.
However, since these upper and lower bounds have a large gap of at least $d_y-1$, they could not achieve the tight minimum width.

{\bf Recent progress.} Recently, \citet{park21} characterize the exact minimum width of $\relu$ networks for universal approximation of $L^p(\mathbb R^{d_x},\mathbb R^{d_y})$: $w_{\min}=\max\{d_x+1,d_y\}$. To bypass width $d_x+d_y$ in the previous constructions and to prove the tight upper bound $w_{\min}\le\max\{d_x+1,d_y\}$, they proposed the {coding scheme} which first {encodes} a $d_x$-dimensional input $x$ to a scalar-valued codeword $c$ %
and {decodes} that codeword to a $d_y$-dimensional vector approximating $f(x)$.
They approximate each of these functions using $\relu$ networks of width $d_x+1$, and $\max\{d_y,2\}$ as in previous constructions, which results in the tight upper bound.
Using a similar construction, they also prove that networks of width $\max\{d_x+2,d_y+1\}$ are dense in $L^p([0,1]^{d_x},\mathbb R^{d_y})$ if an activation function $\sigma$ is continuous, non-polynomial, and $\sigma'(z)\ne0$ for some $z\in\mathbb R$.
For universal approximation of $L^p([0,1]^{d_x},\mathbb R^{d_y})$ using $\text{Leaky-}\relu$ networks, \citet{cai23} characterizes $w_{\min}=\max\{d_x,d_y,2\}$ using the results that continuous $L^p$ functions can be approximated by neural ordinary differential equations (ODEs) \citep{li22} and narrow Leaky-$\relu$ networks can approximate neural ODEs \citep{duan22}. However, except for these two cases, the exact minimum width for universal approximation is still unknown.

One interesting observation made by \citet{park21} is that $\relu$ networks of width $2$ are dense in $L^p(\mathbb R,\mathbb R^{2})$ but not dense in $C([0,1],\mathbb R^{2})$. This shows a gap between minimum widths for $L^p$ and uniform approximations.
\citet{cai23} also suggests a similar dichotomy for leaky-$\relu$ networks when $d_x=1$ and $d_y=2$.
Nevertheless, whether such a dichotomy exists for general activation functions and input/output dimensions is unknown.

\subsection{Summary of results}
In this work, we primarily focus on characterizing the minimum width of fully-connected $\relu$ networks for universal approximation on a compact domain. 
However, our results are not restricted to $\relu$; they extend to general activation functions as summarized below.
\begin{itemize}[leftmargin=15pt]
\item \cref{thm:lp-ub} states that width $\max\{d_x,d_y,2\}$ is necessary and sufficient for $\relu$ networks to be dense in $L^p([0,1]^{d_x},\mathbb R^{d_y})$. Compared to the existing result that the minimum width is $\max\{d_x+1,d_y\}$ when the domain is $\mathbb R^{d_x}$ \citep{park21}, our result shows a gap between minimum widths for $L^p$ approximation on the compact and unbounded domains. To our knowledge, this is the first result showing such a dichotomy.
\item Given the exact minimum width in \cref{thm:lp-ub}, our next result shows that the same $w_{\min}$ holds for the networks using any of $\relul$ activation functions.
Specifically, \cref{cor:general-lp} states that {width $\max\{d_x,d_y,2\}$ is necessary and sufficient} for $\sigma$ networks to be dense in $L^p([0,1]^{d_x},\mathbb R^{d_y})$ if $\sigma$ is in $\{\softplus,$ $\text{Leaky-}\relu, \elu, \celu, \selu\}$, or $\sigma$ is in $\{\gelu,\silu,\mish\}$ and $d_x+d_y\ge3$, which generalizes the previous result for $\text{Leaky-}\relu$ networks \citep{cai23}.
\item 
Our last result improves the previous lower bound on the minimum width for uniform approximation: $w_{\min}\ge d_y+1$ for $\relu$ networks if $d_x=1,d_y=2$. \cref{thm:unif-lb-leakyrelu} states that $\sigma$ networks of width $d_y$ is \emph{not dense} in $C([0,1]^{d_x},\mathbb R^{d_y})$ if $d_x<d_y\le2d_x$ and $\sigma$ can be uniformly approximated by a sequence of continuous injections, e.g., monotone functions such as $\relu$.
\item For uniform approximation using Leaky-$\relu$ networks, the lower bound in \cref{thm:unif-lb-leakyrelu} is tight if $d_y=2d_x$: there is a matching upper bound $\max\{2d_x+1,d_y\}$ \citep{hwang23}.
Furthermore, together with \cref{thm:lp-ub,cor:general-lp}, \cref{thm:unif-lb-leakyrelu} extends the prior observations showing the dichotomy between $L^p$ and uniform approximations \citep{park21,cai23} to general activation functions and input/output dimensions.
\item Our proof techniques also generalize to $L^p$ approximation of sequence-to-sequence functions via recurrent neural networks (RNNs). \cref{thm:rnn_lp,thm:rnn_lp-relulike} in \cref{sec:rnn} show that the same $w_{\min}$ in \cref{thm:lp-ub,cor:general-lp} holds for RNNs. In addition, \cref{thm:brnn_lp} in \cref{sec:rnn} shows~that $w_{\min}\le\max\{d_x,d_y,2\}$ for bidirectional RNNs using $\relu$ or $\relul$ activation functions.
\end{itemize}

\subsection{Organization}
We first introduce notations and our problem setup in \cref{sec:setup}.
In \cref{sec:main}, we formally present our main results and discuss them.
In \cref{sec:lp-ub}, we present the proof of the tight upper bound on the minimum width in \cref{thm:lp-ub}. %
In \cref{sec:pfthm:unif-lb-leakyrelu}, we prove \cref{thm:unif-lb-leakyrelu} by providing a continuous function $f^*:\mathbb R^{d_x}\to\mathbb R^{d_y}$ with $d_x<d_y\le2d_x$
that cannot be uniformly approximated by a width-$d_y$ network using general activation functions. %
Lastly, we conclude the paper in \cref{sec:conclusion}.

\section{Problem setup and notation}\label{sec:setup}

We mainly consider fully-connected neural networks that consist of affine transformations and an activation function.
Given an activation function $\sigma:\mathbb R\to\mathbb R$, we define an $L$-layer neural network $f$ of input and output dimensions $d_x,d_y\in\mathbb N$, and hidden layer dimensions $d_1, \dots, d_{L-1}$ as follows:
\begin{align}\label{eq:def-nn}
    f \defeq t_L \circ \phi_{L-1} \circ \cdots \circ t_2 \circ \phi_1 \circ t_1,
\end{align}
where $t_\ell : \mathbb R^{d_{\ell-1}} \to \mathbb R^{d_\ell}$ is an affine transformation and $\phi_\ell$ is defined as
    $\phi_\ell(x_1,\dots,x_{d_\ell}) = \left( \sigma(x_1), \dots, \sigma(x_{d_\ell}) \right)$ 
for all $\ell\in[L]$.
We denote a neural network $f$ with an activation function $\sigma$ by a ``$\sigma$ network.''
We define the width of $f$ as the maximum over $d_1, \dots, d_{L-1}$.

We say ``$\sigma$ networks of width $w$ are dense in $C(\mathcal X, \mathcal Y)$'' if for any $f^* \in C(\mathcal X, \mathcal Y)$ and $\varepsilon >0$, there exists a $\sigma$ network $f$ of width $w$ such that $\| f^* - f \|_\infty \le \varepsilon$. 
Likewise, we say $\sigma$ networks of width $w$ are dense in $L^p(\mathcal X, \mathcal Y)$ if for any $f^* \in L^p(\mathcal X, \mathcal Y)$ and $\varepsilon >0$, there exists a $\sigma$ network $f$ of width $w$ such that $\| f^* - f \|_p \le \varepsilon$.
We say ``$w_{\min}=w$ for $\sigma$ networks to be dense in $C(\mathcal X,\mathcal Y)$ (or  $L^p(\mathcal X,\mathcal Y)$)'' if $\sigma$ networks of width $w$ are dense in $C(\mathcal X,\mathcal Y)$ (or  $L^p(\mathcal X,\mathcal Y)$) but $\sigma$ networks of width $w-1$ are not.
In other words, $w_{\min}$ denotes the width of neural networks necessary and sufficient for universal approximation in $C(\mathcal X,\mathcal Y)$ (or  $L^p(\mathcal X,\mathcal Y)$).

We lastly introduce frequently used notations.
For $n\in\mathbb N$, we use $\mu_n$ to denote the $n$-dimensional Lebesgue measure and $[n]\defeq\{1,\dots,n\}$.
{For $n\in\mathbb N$ and $\mathcal S\subset\mathbb R^n$, we use $\diam(\mathcal S)\defeq\sup_{x,y\in\mathcal S}\|x-y\|_2$.}
A set $\mathcal H^+\subset\mathbb R^n$ is a half-space if $\mathcal H^+=\{x\in\mathbb R^n:a^\top x+b\ge0\}$ for some $a\in\mathbb R^n\setminus\{(0,\dots,0)\}$ and $b\in\mathbb R$.
A set $\mathcal P \subset \mathbb R^n$ is a (convex) polytope if $\mathcal P$ is bounded and can be represented as an intersection of finite half-spaces.
For $f : \mathbb R^n \to \mathbb R^m$, $f(x)_i$ denotes the $i$-th coordinate of $f(x)$.
We define $\relul$ activation functions ($\relu$, $\text{Leaky-}\relu$, $\gelu$, $\silu$, $\mish$, $\softplus$, $\elu$, $\celu$, $\selu$) in \cref{sec:activation}.
For an activation function with parameters (e.g., Leaky-$\relu$ and $\softplus$), we assume that a single parameter configuration is shared across all activation functions and it is fixed, i.e., we do not tune them when approximating a target function.

\vspace{-0.05in}
\section{Main results}\label{sec:main}
\vspace{-0.05in}
We are now ready to introduce our main results on the minimum width for universal approximation.

{\bf $L^p$ approximation with $\relul$ activation functions.}
Our first result exactly characterizes the minimum width for universal approximation of $L^p([0,1]^{d_x},\mathbb{R}^{d_y})$ using $\relu$ networks.
\begin{theorem}\label{thm:lp-ub}
    $w_{\min}=\max\{d_x, d_y, 2\}$ for $\relu$ networks to be dense in $L^p([0,1]^{d_x},\mathbb{R}^{d_y}\!)$.
\end{theorem}
\cref{thm:lp-ub} states 
that for $\relu$ networks, width $\max\{d_x,d_y,2\}$ is necessary and sufficient for universal approximation of $\smash{L^p([0,1]^{d_x},\mathbb{R}^{d_y})}$. Compared to the existing result that $\relu$ networks of width $\max\{d_x,d_y,2\}$ are not dense in $\smash{L^p(\mathbb R^{d_x},\mathbb{R}^{d_y})}$ if $\smash{d_x+1>d_y\ge2}$ \citep{park21}, \cref{thm:lp-ub} shows a discrepancy between approximating $L^p$ functions on a compact domain (i.e., $[0,1]^{d_x}$) and on the whole Euclidean space (i.e., $\mathbb R^{d_x}$). Namely, a smaller width is sufficient for approximating $L^p$ functions on a compact domain if $d_x+1>d_y$.
We note that a similar result was already known for the \emph{minimum depth} analysis of $\relu$ networks: two-layer $\relu$ networks are dense in $\smash{L^p([0,1]^{d_x},\mathbb{R}^{d_y})}$ but not dense in $L^p(\mathbb R^{d_x},\mathbb{R}^{d_y})$ \citep{lu21,wang22}.

Although \cref{thm:lp-ub} extends the result of \citep{cai23} from Leaky-$\relu$ networks to $\relu$ ones, 
we use a completely different approach for proving the upper bound. \cite{cai23} approximates a target function via a Leaky-$\relu$ network of width $\max\{d_x,d_y,2\}$ using two steps: approximate the target function by a neural ODE first \citep{li22}, then approximate the neural ODE by a Leaky-$\relu$ network \citep{duan22}.
Here, the latter step requires the strict monotonicity of Leaky-$\relu$ and does not generalize to non-strictly monotone activation functions (e.g., $\relu$).

To bypass this issue, we %
carefully analyze the properties of $\relu$ networks and propose a different construction.
In particular, our construction of a $\relu$ network of width $\max\{d_x,d_y,2\}$ that approximates a target function is based on the coding scheme consisting of two functions: an encoder and a decoder.
First, an encoder encodes each input to a scalar-valued codeword, and a decoder maps each codeword to an approximate target value.
\citet{park21} approximate the encoder with the domain $\mathbb R^{d_x}$ using a $\relu$ network of width $d_x+1$ and implemented the decoder using a $\relu$ network of width $\max\{d_y,2\}$ to obtain a universal approximator of width $\max\{d_x+1,d_y\}$ for $\smash{L^p(\mathbb R^{d_x},\mathbb R^{d_y})}$.
By exploiting the compactness of the domain and based on the functionality of $\relu$ networks, we successfully approximate the encoder using a $\relu$ network of width $\max\{d_x,2\}$ and show that $\relu$ networks of width $\max\{d_x,d_y,2\}$ is dense in  $L^p([0,1]^{d_x},\mathbb R^{d_y})$.

The lower bound $w_{\min}\ge\max\{d_x,d_y,2\}$ in \cref{thm:lp-ub} follows from an existing lower bound $w_{\min}\ge\max\{d_x,d_y\}$ \citep{cai23} and a lower bound $w_{\min}\ge2$. Here, the intuition behind each lower bound $d_x,d_y,$ and $2$ is rather straightforward. If a network has width $d_x-1$, then it must have the form $g(Mx)$ for some continuous function $\smash{g:\mathbb R^{d_x-1}\to\mathbb R^{d_y}}$ and $M \in \smash{\mathbb{R}^{(d_x-1) \times d_x}}$, which cannot universally approximate, e.g., consider approximating $\smash{\|x\|_2^2}$. Likewise, if a network has width $d_y-1$, then it must have the form $Nh(x)$ for some continuous function $\smash{h:\mathbb R^{d_x}\to\mathbb R^{d_y-1}}$ and $\smash{N\in\mathbb{R}^{d_y\times (d_y-1)}}$, which cannot universally approximate. Lastly, a $\relu$ network of width $1$ is monotone and hence, cannot approximate non-monotone functions. Combining these three arguments leads us to the lower bound $\max\{d_x,d_y,2\}$ in \cref{thm:lp-ub}. 
We note that our proof techniques are not restricted to $\relu$; they can be extended to various $\relu$-like activation functions.

\begin{theorem}\label{cor:general-lp}
    $w_{\min}=\max\{d_x,d_y,2\}$ for $\varphi$ networks to be dense in $L^p([0,1]^{d_x},\mathbb R^{d_y})$ if $\varphi\in\{\elu$, $\text{\rm Leaky-}\relu, \softplus, \celu, \selu \}$, or $\varphi\in\{\gelu, \silu, \mish\}$ and $d_x+d_y\ge3$.
\end{theorem}

\cref{cor:general-lp} provides that for $\elu$, $\text{Leaky-}\relu$, $\softplus$, $\celu$, and $\selu$ networks, width $\max\{d_x,d_y,2\}$ is necessary and sufficient for universal approximation of $\smash{L^p([0,1]^{d_x},\mathbb{R}^{d_y})}$. 
On the other hand, the minimum width of $\gelu, \silu,$ and $\mish$ networks to be dense in $\smash{L^p([0,1]^{d_x},\mathbb{R}^{d_y})}$ is $\max\{d_x,d_y,2\}$ if $d_x+d_y\ge3$.
In particular, \cref{cor:general-lp} can be further generalized to any continuous function $\rho$ such that $\relu$ can be uniformly approximated by a $\rho$ network of width one on any compact domain, within an arbitrary uniform error.
We present the proof for the upper bound $w_{\min}\le\max\{d_x,d_y,2\}$ in \cref{thm:lp-ub} in \cref{sec:lp-ub} while the proof for the matching lower bound in \cref{thm:lp-ub} and the proof of \cref{cor:general-lp} are deferred to \cref{sec:pf:general-lb} and \cref{sec:pfcor:general-lp-ub}.

We note that our proof techniques easily extend to RNNs and bidirectional RNNs: \cref{thm:rnn_lp,thm:rnn_lp-relulike} in \cref{sec:rnn} shows that the same result in \cref{thm:lp-ub,cor:general-lp} also holds for RNNs. Furthermore, \cref{thm:brnn_lp} in \cref{sec:rnn} shows that $w_{\min}\le\max\{d_x,d_y,2\}$ for bidirectional RNNs using any of $\relu$ or $\relul$ activation functions to be dense in $\smash{L^p([0,1]^{d_x},\mathbb R^{d_y})}$.

{\bf Uniform approximation with general activation functions.} 
For $\relu$ networks, it is known that $w_{\min}$ for $C([0,1]^{d_x},\mathbb R^{d_y})$ is greater than that for $L^p([0,1]^{d_x},\mathbb R^{d_y})$ in general.
This is shown by the observation in \citep{park21}: if $d_x=1$ and $d_y=2$, 
then width $2$ is sufficient for $\relu$ networks to be dense in $L^p([0,1]^{d_x},\mathbb R^{d_y})$, but insufficient to be dense in  $C([0,1]^{d_x},\mathbb R^{d_y})$.
However, whether this observation extends %
has been unknown. 
Our next theorem shows that a similar result holds for a wide class of activation functions and $d_x,d_y$. 
\begin{theorem}\label{thm:unif-lb-leakyrelu}
For any continuous $\varphi:\mathbb R\to\mathbb R$ that can be uniformly approximated by a sequence of continuous injections, if $d_x<d_y\le2d_x$, then $w_{\min}\ge d_y+1$ for $\varphi$ networks to be dense in $C([0,1]^{d_x},\mathbb{R}^{d_y})$. 
\end{theorem}
\cref{thm:unif-lb-leakyrelu} states that $w_{\min}\ge d_y+1$ for $C([0,1]^{d_x},\mathbb R^{d_y})$ if $d_y\in(d_x,2d_x]$ and the activation function can be uniformly approximated by a sequence of continuous injections (e.g., any monotone continuous function such as $\relu$).
This bound is tight for Leaky-$\relu$ networks if $2d_x=d_y$, together with the matching upper bound $\max\{2d_x+1,d_y\}$ on $w_{\min}$ \citep{hwang23}.
Combined with \cref{thm:lp-ub}, this result implies that $w_{\min}$ for $\relu$ networks to be dense in $C([0,1]^{d_x},\mathbb R^{d_y})$ is strictly larger than that for $L^p([0,1]^{d_x},\mathbb R^{d_y})$ if $d_x<d_y\le2d_x$. 
With \cref{cor:general-lp} and monotone $\relul$ activation functions, a similar observation can also be made. %

We prove \cref{thm:unif-lb-leakyrelu} by explicitly constructing a continuous target function that cannot be approximated by a $\varphi$ network of width $d_y$ in a small uniform distance where $\varphi$ is an activation function that can be uniformly approximated by a sequence of continuous one-to-one functions.
In particular, based on topological arguments, we prove that any continuous function that uniformly approximates our target function within a small error has an intersection, i.e., it cannot be uniformly approximated by injective functions, which leads us to the statement of \cref{thm:unif-lb-leakyrelu}.
We present a detailed proof of \cref{thm:unif-lb-leakyrelu} including the formulation of our target function in \cref{sec:pfthm:unif-lb-leakyrelu}.

We lastly note that all results with the domain $[0,1]^{d_x}$ also hold for arbitrary compact domain $\mathcal K\subset\mathbb R^{d_x}$: if the target function $f^*$ is continuous, then one can always find $K>0$ such that $\mathcal K\subset[-K,K]^{d_x}$ and continuously extend $f^*$ to $[-K,K]^{d_x}$ by the Tietze extension lemma \citep{munkres}. If $f^*$ is $L^p$, then approximate $f^*$ by some continuous function and perform the extension.

\section{Tight upper bound on minimum width for $L^p$-approximation}\label{sec:lp-ub}

In this section, we prove the upper bound in \cref{thm:lp-ub} by explicitly constructing a $\relu$ network of width $\max\{d_x,d_y,2\}$ approximating a target function in $L^p([0,1]^{d_x},\mathbb R^{d_y})$. 
Since continuous functions on $[0,1]^{d_x}$ are dense in $L^p([0,1]^{d_x},\mathbb R^{d_y})$ \citep{rudin}, it suffices to prove the following lemma to show the upper bound. 
Here, we restrict the codomain to $[0,1]^{d_y}$; however, this result can be easily extended to the codomain $\mathbb R^{d_y}$ since the range of $f^*\in C([0,1]^{d_x},\mathbb R^{d_y})$ is compact. 
\begin{lemma}\label{lem:ub-lp}
Let $\varepsilon>0$, $p\ge1$, and $f^*\in C([0,1]^{d_x},[0,1]^{d_y})$. Then, there exists a $\relu$ network $f:[0,1]^{d_x}\to\mathbb R^{d_y}$ of width $\max\{d_x,d_y,2\}$ such that $\|f-f^*\|_{p}\le\varepsilon$.
\end{lemma}

\subsection{Coding scheme and $\relu$ network implementation (proof of \cref{lem:ub-lp})}\label{sec:coding}
\begin{figure}
    \centering
    \includegraphics[width=0.9\linewidth]{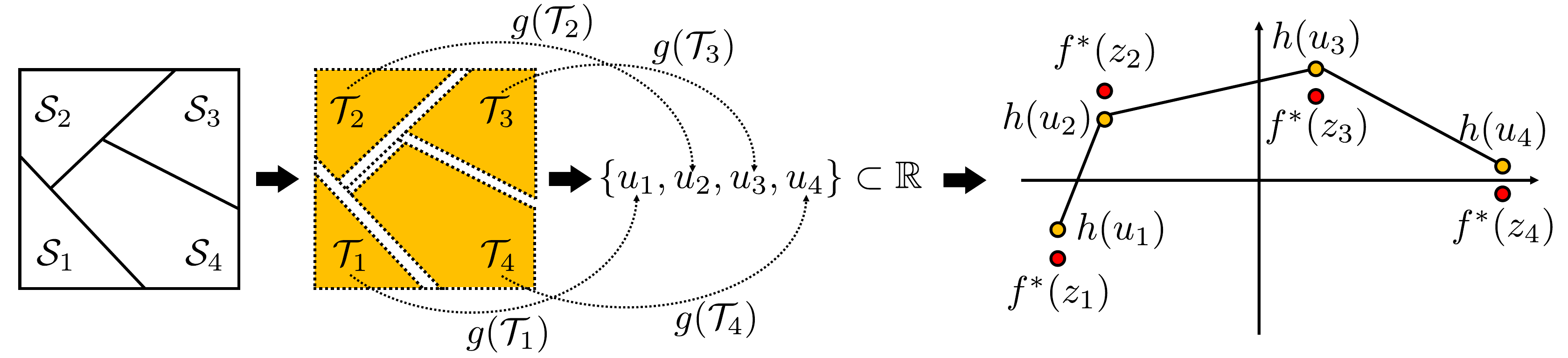}
    \caption{Illustration of our encoder and decoder when $d_x=2$, $d_y=1$ and $k=4$. Our encoder $g$ first maps each element of $\{\mathcal T_1,\dots,\mathcal T_4\}$ %
    to distinct scalar codewords $u_1,\dots,u_4$. Then, the decoder $h$ maps each codeword $u_i$ to $h(u_i)\approx f^*(z_i)$ {for some $z_i \in \mathcal T_i$.}
    }
    \label{fig:coding}
    \vspace{-0.1in}
\end{figure}

Our proof of \cref{lem:ub-lp} is based on the \emph{coding scheme} that consists of two functions \citep{park21}: the \emph{encoder} and \emph{decoder}.
The \emph{encoder} first transforms each input vector $x\in[0,1]^{d_x}$ to a scalar-valued codeword containing the information of $x$; then the \emph{decoder} maps each codeword to a target vector in $[0,1]^{d_y}$ which approximates $f^*(x)$. Namely, the composition of these two functions approximates the target function.
The precise operations of our encoder and decoder are as follows.

Suppose that a partition $\{\mathcal S_1,\dots,\mathcal S_k\}$ of the domain $[0,1]^{d_x}$ is given. %
Then, the encoder maps each input vector in $\mathcal S_i$ to some scalar-valued codeword $c_i$. Here, if the diameter of the set $\mathcal S_i$ is small enough, then it is reasonable to map vectors in $\mathcal S_i$ to the same codeword, say $c_i\in\mathbb R$, since the target function $f^*$ is uniformly continuous on $[0,1]^{d_x}$, i.e., $f^*(x)\approx f^*(x')$ for all $x,x'\in\mathcal S_i$.
However, since such an encoder is discontinuous in general, we approximate it using a $\relu$ network via the following lemma. 
We present the main proof idea of \cref{lem:encoder} in \cref{sec:pfsketch:encoder} and defer the full proof to \cref{sec:pflem:encoder}.
\begin{lemma}\label{lem:encoder}
    For any $\alpha,\beta>0$, there exist disjoint measurable sets $\mathcal T_1,\dots,\mathcal T_k\subset[0,1]^{d_x}$ and
    a $\relu$ network $f:\mathbb R^{d_x}\to\mathbb R$ of width $\max\{d_x,2\}$ such that 
    \begin{itemize}[leftmargin=15pt]
        \item $\diam(\mathcal T_i)\le\alpha$ for all $i\in[k]$,
        \item $\mu_{d_x}\big(\bigcup_{i=1}^k\mathcal T_i\big)\ge1-\beta$, and
        \item $f(\mathcal T_i)=\{c_i\}$ for all $i\in[k]$, for some distinct $c_1,\dots,c_k\in\mathbb R$.
    \end{itemize}
\end{lemma}
\cref{lem:encoder} states that there is an (approximate) encoder given by a $\relu$ network $g$ of width $\max\{d_x,2\}$ that can assign distinct codewords to $\mathcal T_1,\dots,\mathcal T_k$.
Here, $\mathcal T_1,\dots,\mathcal T_k$ can be considered as an {approximate} partition since
they are disjoint and cover at least $1-\beta$ fraction of the domain for any $\beta>0$. 
By choosing a small enough $\alpha$, we can have a small \emph{information loss} of the input vectors in $\mathcal T_1,\dots,\mathcal T_k$, incurred by encoding them via \cref{lem:encoder}.
We note that such an approximate encoder may map inputs that are not contained in $\mathcal T_1\cup\cdots\cup\mathcal T_k$ to arbitrary values.

Once the encoder transforms 
{all input vectors in
$\mathcal T_i$
to a {single} codeword $c_i\in\mathbb R$, 
the decoder maps the codeword to a $d_y$-dimensional vector 
that approximates $f^*(\mathcal T_i)$. %
We implement the decoder using a $\relu$ network using the following lemma, which is a corollary of Lemma~9 and Lemma~10 in \citet{park21}. See \cref{sec:pflem:decoder0} for its formal derivation.
\begin{lemma}\label{lem:decoder0}
    For any $p\ge1$, $\gamma>0$, distinct $c_1,\dots,c_k\in\mathbb R$, and $v_1,\dots,v_k\in\mathbb R^{d_y}$, there exists a $\relu$ network $f:\mathbb R\to[0,1]^{d_y}$ of width $\max\{d_y,2\}$ such that $\|f(c_i)-v_i\|_p\le\gamma$ for all $i\in[k]$.
\end{lemma}
\cref{lem:ub-lp} follows from an (approximate) encoder and decoder in \cref{lem:encoder,lem:decoder0}.
Let $\mathcal T_1,\dots,\mathcal T_k$ be an approximate partition and $g$ be a $\relu$ network of width $\max\{d_x,2\}$ in \cref{lem:encoder} with some $\alpha,\beta>0$.
Likewise, let $h$ be a $\relu$ network of width $\max\{d_y,2\}$ in \cref{lem:decoder0} with codewords $c_1,\dots,c_k$ generated by $g$, $v_i=f^*(z_i)$ for some $z_i\in\mathcal T_i$ for all $i\in[k]$, and some $\gamma>0$. Then, $f=h\circ g$ can be implemented by a $\relu$ network of width $\max\{d_x,d_y,2\}$ and can approximate the target function $f^*$ in $\varepsilon$ error if we choose small enough $\alpha,\beta,\gamma$.
Since the codomain of our decoder is $[0,1]$, one can observe that $f(x)\in[0,1]$ for any $x$ that is not contained in any of $\mathcal T_1,\dots,\mathcal T_k$, i.e., they only incur a small $L^p$ error if $\beta$ is small enough.
\cref{fig:coding} illustrates our encoder and decoder construction.
See \cref{sec:choice-abc} for our choices of $\alpha,\beta,\gamma$ achieving the statement of \cref{lem:ub-lp}.

\subsection{Approximating encoder using $\relu$ network (proof sketch of \cref{lem:encoder})}\label{sec:pfsketch:encoder}

In this section, we sketch the proof of \cref{lem:encoder} where the full proof is in \cref{sec:pflem:encoder}.
To this end, we first introduce the following key lemma. The proof of \cref{lem:tool1} is deferred to \cref{sec:pflem:tool1}
\begin{lemma}\label{lem:tool1}
For any $d_x\in\mathbb N$, a compact set $\mathcal K\subset\mathbb R^{d_x}$, $a,c\in\mathbb R^{d_x}$ such that $a^\top c>0$, and $b\in\mathbb R$, there exists a two-layer $\relu$ network $f:\mathcal K\to\mathbb R^n$ of width $d_x$ such that
\begin{align*}
f(x)=\begin{cases}
x~&\text{if}~a^\top x+b\ge 0\\
x-\frac{a^\top x+b}{a^\top c}\times c~&\text{if}~a^\top x+b<0
\end{cases}.
\end{align*}
\end{lemma}
\cref{lem:tool1} states that 
there exists a two-layer $\relu$ network of width $d_x$ on a compact domain
that preserves the points in the half-space $\mathcal H^+=\{x\in\mathbb R^{d_x}:a^\top x+b\ge0\}$ and projects points not in $\mathcal H^+$ to the boundary of $\mathcal H^+$ along the direction determined by a vector $c$ as illustrated in \cref{fig:proof-proj-a}.

This lemma has two important applications. First, for any bounded set, \cref{lem:tool1} enables us to project it onto a hyperplane (the boundary of $\mathcal H^+)$, along a vector $c$.
In other words, we can use \cref{lem:tool1} for decreasing a dimension of a bounded set or moving a point as illustrated in \cref{fig:proof-proj-a}.
Furthermore, given a polytope $\mathcal P\subset\mathbb R^{d_x}$ and a half-space $\mathcal H^+$ such that both $\mathcal P\cap\mathcal H^+$ and $\mathcal P\setminus\mathcal H^+$ are non-empty, we can construct a $\relu$ network of width $d_x$ that preserves points in $\mathcal P\cap\mathcal H^+$ and maps some $\mathcal T\subset\mathcal P\setminus\mathcal H^+$ with $\mu_{d_x}(\mathcal T)\approx\mu_{d_x}(\mathcal P\setminus\mathcal H^+)$ to a single point disjoint to $\mathcal P\cap\mathcal H^+$.
As illustrated in \cref{fig:proof-proj-b}, this can be done by mapping $\mathcal P\setminus\mathcal H^+$ onto the boundary of $\mathcal H^+$ first, and then, iteratively projecting a subset of the image of $\mathcal P\setminus\mathcal H^+$ (i.e., the image of $\mathcal T$) to a single point using \cref{lem:tool1}. We note that the measure of $\mathcal T$ can be arbitrarily close to that of $\mathcal P\setminus\mathcal H^+$ by choosing a proper $c$ in \cref{lem:tool1} when projecting $\mathcal P\setminus\mathcal H^+$ onto the boundary of $\mathcal H^+$. %

\begin{figure}
     \centering
     \begin{subfigure}[b]{0.2295\textwidth}%
         \centering
         \includegraphics[width=\linewidth]{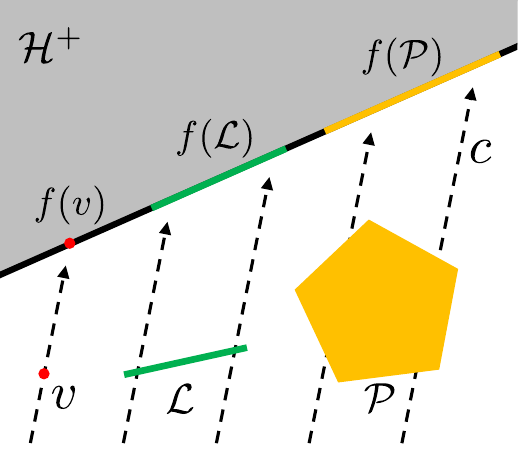}
         \caption{}
         \label{fig:proof-proj-a}
     \end{subfigure}
     \quad
     \begin{subfigure}[b]{0.63\textwidth}%
         \centering
         \includegraphics[width=\linewidth]{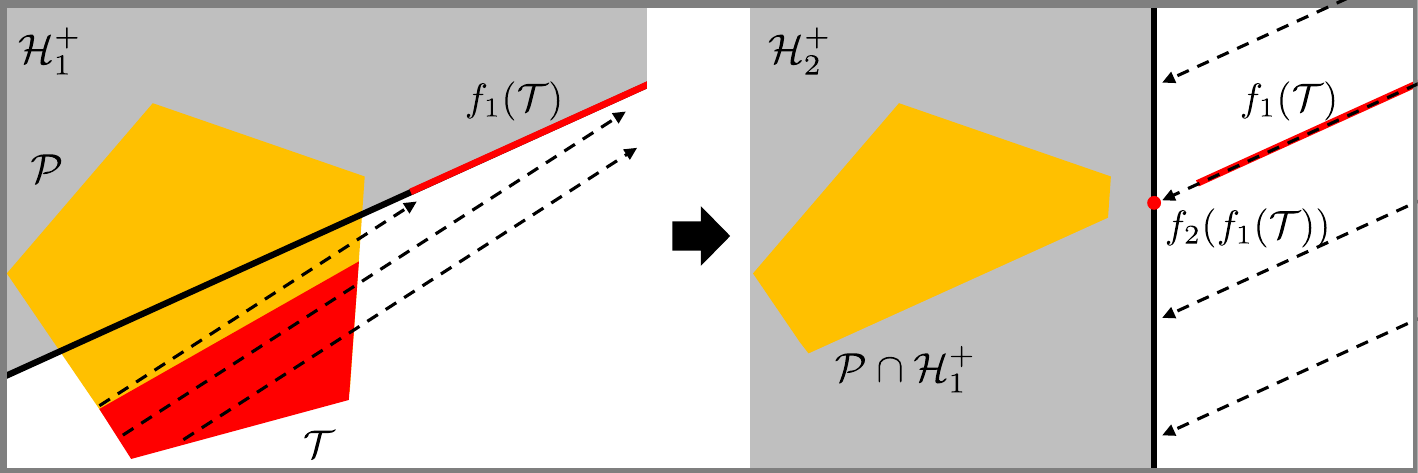}
         \caption{}
         \label{fig:proof-proj-b}
     \end{subfigure}
        \caption{Construction of $f$ in \cref{lem:tool1}. (a) $f$ preserves points in the half-space $\mathcal H^+$ represented by the gray area and projects points outside of $\mathcal H^+$ to the boundary of $\mathcal H^+$.
        (b) Illustrations of mapping $\mathcal T$ to a single point disjoint to $\mathcal P \cap \mathcal H_1^+$ when $d_x=2$:
        $f_1$ maps $\mathcal T$ onto the boundary of $\mathcal H_1^+$ and then $f_2$ maps $f_1(\mathcal T)$ to the point $f_2(f_1(\mathcal T))$ while preserving points in $\mathcal P \cap \mathcal H_1^+$.}
        \label{fig:proof-proj}
        \vspace{-0.1in}
\end{figure}

We now describe our construction of $f$ in \cref{lem:encoder}.
First, suppose that there is a partition $\{\mathcal S_1,\dots,\mathcal S_k\}$ of the domain $[0,1]^{d_x}$
where each $\mathcal S_i$ can be represented as
\begin{align*}
\mathcal S_i=[0,1]^{d_x}\cap\bigg(\bigcap_{j=1}^{i-1} \mathcal H^+_j\bigg)\cap(\mathcal H^+_i)^c
\end{align*} 
for some half-spaces $\mathcal H^+_1,\dots,\mathcal H^+_k$; see the first image in \cref{fig:enc} for example.
Suppose further that $\diam(\mathcal S_i)\le\alpha$ for all $i\in[k]$. We note that such a partition always exists as stated in \cref{lem:poly-cut} in \cref{sec:pflem:encoder}.
As in the second image of \cref{fig:enc}, the most part of $\mathcal S_1$ ($\mathcal T_1$ in \cref{fig:enc}) can be mapped into a single point ($\smash{u_1}$ in \cref{fig:enc}) disjoint to $\smash{\mathcal S_2\cup\cdots\cup\mathcal S_k}$ using \cref{lem:tool1}.
Likewise, we map the most part of $\mathcal S_2$ ($\mathcal T_2$ in \cref{fig:enc}) to a single point  disjoint to $\smash{\{u_1\}\cup\mathcal S_3\cup\cdots\cup\mathcal S_k}$. Here, if $u_1\notin\mathcal H_2^+$, we first move it so that $u_1\in\mathcal H_2^+$ using \cref{lem:tool1} while preserving points in $\smash{\mathcal S_2\cup\cdots\cup\mathcal S_k}$. By repeating this procedure, we can consequently map the most parts ($\mathcal T_1,\dots,\mathcal T_k$) of $\mathcal S_1,\dots,\mathcal S_k$ to $k$ distinct points via a $\relu$ network of width $d_x$ (the second last image in \cref{fig:enc}).
We finally project these points to distinct scalar values as illustrated in the last image in \cref{fig:enc}.
See \cref{lem:basic-encoder,lem:distinct-innerprod} in \cref{sec:pflem:encoder} and its proof for the formal statements.

We note that our construction of a $\relu$ network satisfies the three conditions in \cref{lem:encoder}.
The first condition is naturally satisfied since $\mathcal T_i\subset\mathcal S_i$ and $\diam(\mathcal S_i)\le\alpha$.
The second condition can also be satisfied since the measure of $\mathcal T_i$ can be arbitrarily close to that of $\mathcal S_i$ for all $i\in[k]$.
Lastly, our construction maps each $\mathcal T_i$ to a distinct scalar value; this provides the third condition.

\section{Lower bound on minimum width for uniform approximation}\label{sec:pfthm:unif-lb-leakyrelu}
\begin{figure}
    \centering
    \includegraphics[width=0.9\linewidth]{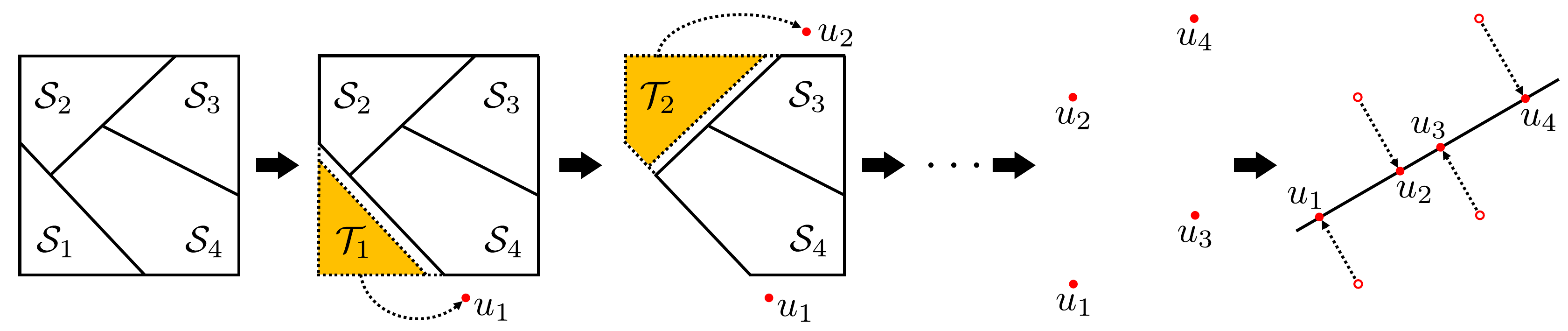}
    \caption{Illustration of the encoder when $d_x=2$ and $k=4$.
    For each partition $\mathcal S_i$, the encoder maps $\mathcal T_i\subset\mathcal S_i$ to some point $\smash{u_i}\notin\{u_1,\dots,u_{i-1}\}\cup\mathcal S_{i+1}\cup\cdots\cup\mathcal S_k$.
    After that, the encoder maps $u_1,\dots,u_k$ to some distinct scalar values by projecting them.
    }
    \label{fig:enc}
    \vspace{-0.1in}
\end{figure}

In this section, we prove \cref{thm:unif-lb-leakyrelu} by explicitly showing the existence of a continuous function $f^*:[0,1]^{d_x}\to\mathbb R^{d_y}$ that cannot be approximated by any $\varphi$ network of width $d_y$ within $1/3$ error in the uniform norm when $d_x<d_y\le2d_x$. Based on the following lemma, we assume that the activation function $\varphi$ is a continuous injection throughout the proof without loss of generality. The proof of \cref{lem:inj-approx} is presented in \cref{sec:pflem:inj-approx}.

\begin{lemma}\label{lem:inj-approx}
    Let $\sigma:\mathbb R\to\mathbb R$ be a continuous function that can be uniformly approximated by a sequence of continuous injections.
    Then, for any $\sigma$ network $f:[0,1]^{d_x}\to\mathbb R^{d_y}$ of width $w$ and for any $\varepsilon>0$, there exists a $\varphi$ network $g:[0,1]^{d_x}\to\mathbb R^{d_y}$ of width $w$ such that $\varphi:\mathbb R\to\mathbb R$ is a continuous injection and $\|f-g\|_\infty\le\varepsilon$. %
\end{lemma}

{\bf Our choice of $f^*$.} We consider $f^*$ of the following form: for $r=d_y-d_x$, $x=(x_1,\dots,x_{d_x})\in[0,1]^{d_x}$, $\mathcal D_1=[0,1/3]^{d_x}$, and $\mathcal D_2=[2/3,1]^r\times\{1\}^{d_x-r}$,
\begin{align*}
    f^*(x)=
    \begin{cases}
        (1-6x_1,1-6x_2,\dots,1-6x_{d_x},0,\dots,0)~&\text{if}~x\in\mathcal D_1\\
        (0,\dots,0,6x_1-5,6x_2-5,\dots,6x_r-5)~&\text{if}~x\in\mathcal D_2\\
        g^*(x)~&\text{otherwise}
    \end{cases},
\end{align*}
where $g^*$ is some continuous function that makes $f^*$ continuous; such $g^*$ always exists by 
{the Tietze extension lemma and the pasting lemma \citep{munkres}}.
We note that $f^*|_{\mathcal D_1}$ and $f^*|_{\mathcal D_2}$ are injections whose images are $[-1,1]^{d_x}\times\{0\}^r$ and $\{0\}^{d_x}\times[-1,1]^{r}$, respectively, i.e., $f^*(\mathcal D_1)\cap f^*(\mathcal D_2)=\{(0,\dots,0)\}$.
See \cref{fig:proof-lb-a,fig:proof-lb-b} for illustrations of $\mathcal D_1,\mathcal D_2,f^*(\mathcal D_1)$, and $f^*(\mathcal D_2)$.

{\bf Assumptions on $\varphi$ network approximating $f^*$.} Suppose for a contradiction that there are a continuous injection $\varphi:\mathbb R\to\mathbb R$ and a $\varphi$ network $f:[0,1]^{d_x}\to\mathbb R^{d_y}$ such that $\|f^*-f\|_\infty\le1/3$~and
$f=t_L\circ\phi\circ\cdots\circ t_2\circ\phi\circ t_1$
where $t_1:\mathbb R^{d_x}\to\mathbb R^{d_y}$, $t_2,\dots,t_{L}:\mathbb R^{d_y}\to\mathbb R^{d_y}$ are some affine transformations and $\phi:\mathbb R^{d_y}\to\mathbb R^{d_y}$ is a pointwise application of $\varphi$.
Without loss of generality, we assume that
$t_2,\dots,t_{L}$ are invertible, as invertible affine transformations are dense in the space of affine transformations on bounded support, endowed with $\|\cdot\|_\infty$.
Likewise, we assume that $t_1$ is injective.
Since $\varphi$ is an injection, $f$ is also an injection, i.e., 
\begin{align}
    f(\mathcal D_1)\cap f(\mathcal D_2)=\emptyset.\label{eq:pfthm:unif-lb}
\end{align}
However, one can expect that such $f$ cannot be injective as illustrated in \cref{fig:proof-lb-c}. Based on this intuition, we now formally show a contradiction.

{\bf Proof by contradiction.} Define $h_1:[-1,1]^{d_x}\to\mathbb R^{d_y}$, $h_2:[-1,1]^{r}\to\mathbb R^{d_y}$, and $\psi:[-1,1]^{d_y}\to[-1,1]^{d_y}$ as follows: for $\alpha\in[-1,1]^{d_x}$, $\beta\in[-1,1]^{r}$, and $\gamma=(\alpha,\beta)\in[-1,1]^{d_y}$,
\begin{align*}
h_1(\alpha)=f\big((\alpha+1)/6\big),~~~
h_2(\beta)=f\big((\beta+5)/6,1,\dots,1\big),~~~
\psi(\gamma)=\frac{h_1(\alpha)-h_2(\beta)}{\|h_1(\alpha)-h_2(\beta)\|_\infty}.
\end{align*}
From the definitions of $h_1,h_2$ and \cref{eq:pfthm:unif-lb}, one can observe that
\begin{align*}
h_1([-1,1]^{d_x})\cap h_2([-1,1]^{r})=f(\mathcal D_1)\cap f(\mathcal D_2)=\emptyset,
\end{align*}
i.e., $\psi$ is well-defined (and continuous) as $\|h_1(\alpha)-h_2(\beta)\|_\infty>0$ for all $(\alpha,\beta)\in[-1,1]^{d_y}$.
From the definitions of $\psi$ and the infinity norm, it holds that 
\begin{itemize}[leftmargin=17pt]
\item[\it(i)] $\psi([-1,1]^{d_y})\subset[-1,1]^{d_y}$ and
\item[\it(ii)] for each $\gamma \in[-1,1]^{d_y}$, there exists $i\in[d_y]$ such that $\psi(\gamma)_i\in\{-1,1\}$. %
\end{itemize}
By {\it (i)}, continuity of $\psi$, and the Brouwer's fixed point theorem (\cref{lem:brouwer}), there exists $\gamma^*=(\alpha^*,\beta^*)\in[-1,1]^{d_y}$ such that $\psi(\gamma^*)=\gamma^*$.
Furthermore, by {\it (ii)}, there should be $i^*\in[d_y]$ such that $\gamma_{i^*}^*\in\{-1,1\}$.
However, we now show that such $i^*$ does not exist.

Suppose that there is $i^*\in[d_x]$ satisfying $\gamma_{i^*}^*=\psi(\gamma^*)_{i^*}=1$.
Then, by the definitions of $f,h_1,h_2$ and the assumption $\|f^*-f\|_\infty\le1/3$, the following holds for all $z=(x,y)\in[-1,1]^{d_y}$ with $z_{i^*}=1$: $h_1(x)_{i^*}\in[-4/3,-2/3]$ %
and $h_2(y)_{i^*}\in[-1/3,1/3]$. 
This implies that ${h_1(x)_{i^*}-h_2(y)_{i^*}}$ is negative for all $(x,y)\in[-1,1]^{d_y}$, and therefore, $\psi(\gamma^*)_{i^*}<0$ which contradicts $\psi(\gamma^*)_{i^*}=1$.
Using similar arguments, one can also show the contradiction for the two remaining cases: $i^*\in[d_x]$ and $\gamma_{i^*}^*=-1$; and $i^*\in[d_y]\setminus[d_x]$ and $\gamma_{i^*}^*\in\{-1,1\}$. In other words, $\gamma^*_i$ cannot be any of $\{-1,1\}$ for all $i\in[d_y]$.
This contradicts {\it (ii)} and proves \cref{thm:unif-lb-leakyrelu}.

\begin{figure}
     \centering
     \begin{subfigure}[b]{0.25\textwidth}
         \centering
         \includegraphics[width=\linewidth]{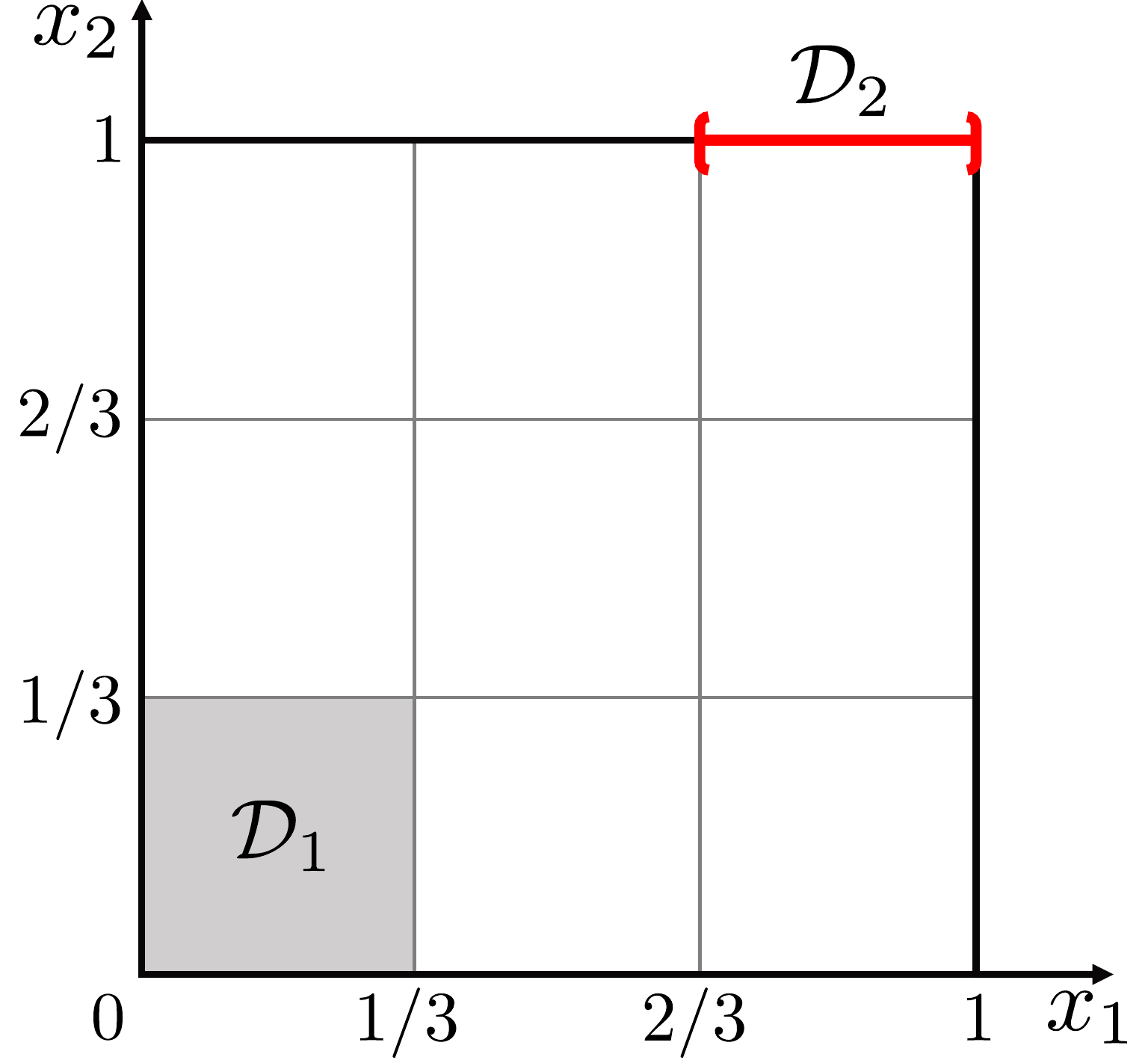}
         \caption{}
         \label{fig:proof-lb-a}
     \end{subfigure}
     \qquad
     \begin{subfigure}[b]{0.25\textwidth}
         \centering
         \includegraphics[width=\linewidth]{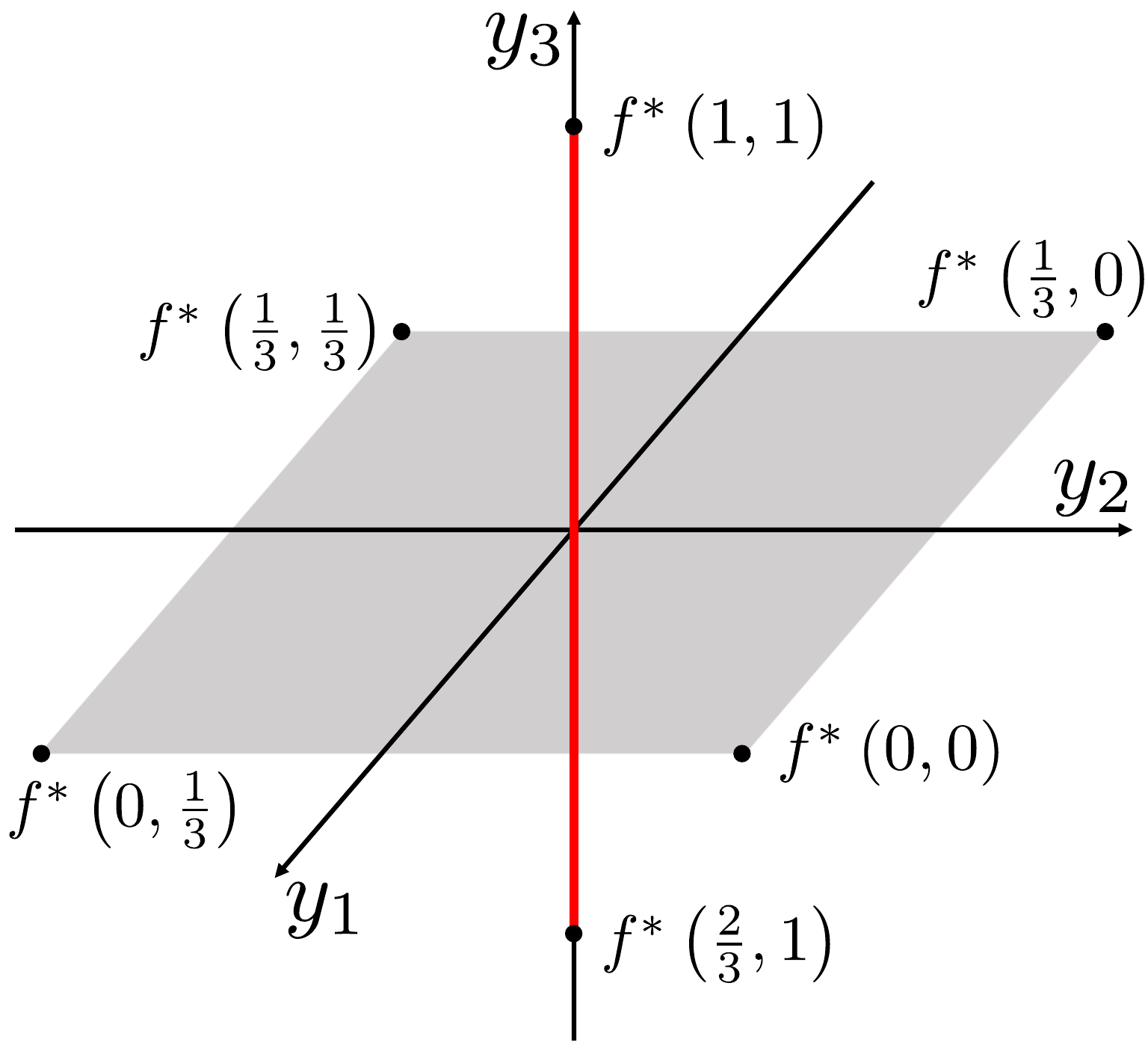}
         \caption{}
         \label{fig:proof-lb-b}
     \end{subfigure}
     \qquad
     \begin{subfigure}[b]{0.25\textwidth}
         \centering
         \includegraphics[width=\linewidth]{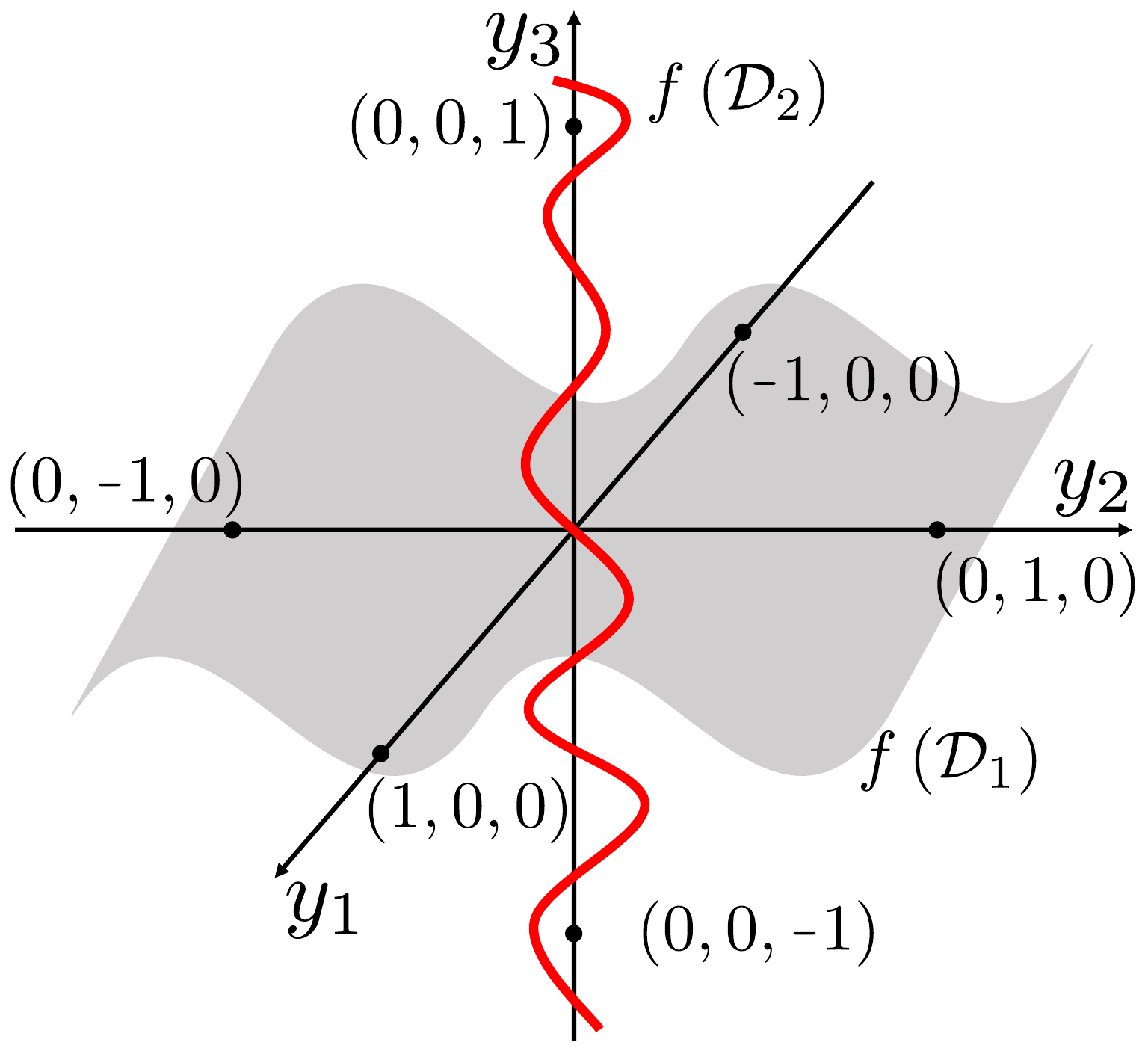}
         \caption{}
         \label{fig:proof-lb-c}
     \end{subfigure}
        \caption{%
        $\mathcal D_1,\mathcal D_2$ and their corresponding images of $f^*$ when $d_x=2$ and $d_y=3$ are illustrated by the grey squares and red lines in (a) and (b). One of the possible images of $f(\mathcal D_1)$ and $f(\mathcal D_2)$ are represented by the grey surface and red curve in (c).
        }
        \label{fig:proof-lb}
        \vspace{-0.1in}
\end{figure}

\begin{lemma}[Brouwer's fixed-point theorem~\citep{brouwer}]\label{lem:brouwer}
For any non-empty compact convex set $\mathcal K\subset\mathbb R^n$ and continuous function $f : \mathcal K \to \mathcal K$, there exists $x^* \in \mathcal K$ such that $f(x^*) = x^*$.
\end{lemma}

\section{Conclusion}\label{sec:conclusion}

Identifying the universal approximation property of deep neural networks is a fundamental problem in the theory of deep learning. Several works have tried to characterize the minimum width enabling universal approximation; however, only a few of them succeed in finding the exact minimum width.
In this work, we first prove that the minimum width of networks using $\relu$ or $\relul$ activation functions is $\max\{d_x,d_y,2\}$ for universal approximation in $L^p([0,1]^{d_x},\mathbb R^{d_y})$. 
Compared to the existing result that width $\max\{d_x+1,d_y\}$ is necessary and sufficient for $\relu$ networks to be dense in $L^p(\mathbb R^{d_x},\mathbb R^{d_y})$, our result shows a dichotomy between universal approximation on a compact domain and the whole Euclidean space.
Furthermore, using a topological argument, we improve the lower bound on the minimum width for uniform approximation when the activation function can be uniformly approximated by a sequence of continuous one-to-one functions: the minimum width is at least $d_y+1$ if $d_x<d_y\le2d_x$, which is shown to be tight for Leaky-$\relu$ networks if $2d_x=d_y$.
This generalizes prior results showing a gap between $L^p$ and uniform approximations to general activation functions and input/output dimensions. 
We believe that our results and proof techniques can help better understand the expressive power of deep neural networks. 

\subsubsection*{Acknowledgements}
NK and SP were supported by Institute of Information \& communications Technology Planning \& Evaluation (IITP) grant funded by the Korea government (MSIT) (No. 2019-0-00079, Artificial Intelligence Graduate
School Program, Korea University) and Basic Science Research Program through the National Research Foundation of Korea (NRF) funded by the Ministry of Education (2022R1F1A1076180).

\bibliography{reference-a2}

\begin{thebibliography}{30}
\providecommand{\natexlab}[1]{#1}
\providecommand{\url}[1]{\texttt{#1}}
\expandafter\ifx\csname urlstyle\endcsname\relax
  \providecommand{\doi}[1]{doi: #1}\else
  \providecommand{\doi}{doi: \begingroup \urlstyle{rm}\Url}\fi

\bibitem[Baum(1988)]{baum88}
Eric~B. Baum.
\newblock On the capabilities of multilayer perceptrons.
\newblock \emph{Journal of Complexity}, 1988.

\bibitem[Boyd and Vandenberghe(2004)]{boyd04}
Stephen~P. Boyd and Lieven Vandenberghe.
\newblock \emph{Convex optimization}.
\newblock Cambridge university press, 2004.

\bibitem[Cai(2023)]{cai23}
Yongqiang Cai.
\newblock Achieve the minimum width of neural networks for universal
  approximation.
\newblock In \emph{International Conference on Learning Representations
  (ICLR)}, 2023.

\bibitem[Cybenko(1989)]{cybenko89}
G.~Cybenko.
\newblock {Approximation by superpositions of a sigmoidal function}.
\newblock \emph{Mathematics of Control, Signals, and Systems (MCSS)},
  2\penalty0 (4):\penalty0 303--314, 1989.

\bibitem[Daniely(2017)]{daniely17}
Amit Daniely.
\newblock Depth separation for neural networks.
\newblock In \emph{Conference on Learning Theory (COLT)}, 2017.

\bibitem[Duan et~al.(2022)Duan, Li, Ji, and Cai]{duan22}
Yifei Duan, Li'ang Li, Guanghua Ji, and Yongqiang Cai.
\newblock Vanilla feedforward neural networks as a discretization of dynamic
  systems.
\newblock \emph{arXiv preprint arXiv:2209.10909}, 2022.

\bibitem[Eldan and Shamir(2016)]{eldan16}
Ronen Eldan and Ohad Shamir.
\newblock The power of depth for feedforward neural networks.
\newblock In \emph{Conference on Learning Theory (COLT)}, 2016.

\bibitem[Florenzano(2003)]{brouwer}
M.~Florenzano.
\newblock \emph{General Equilibrium Analysis: Existence and Optimality
  Properties of Equilibria}.
\newblock Springer US, 2003.

\bibitem[Hanin and Sellke(2017)]{hanin17}
Boris Hanin and Mark Sellke.
\newblock Approximating continuous functions by {ReLU} nets of minimal width.
\newblock \emph{arXiv preprint arXiv:1710.11278}, 2017.

\bibitem[Hornik et~al.(1989)Hornik, Stinchcombe, and White]{hornik89}
K.~Hornik, M.~Stinchcombe, and H.~White.
\newblock Multilayer feedforward networks are universal approximators.
\newblock \emph{Neural Networks}, 2\penalty0 (5):\penalty0 359--366, 1989.

\bibitem[Huang and Babri(1998)]{huang98}
Guang-Bin Huang and Haroon~A Babri.
\newblock Upper bounds on the number of hidden neurons in feedforward networks
  with arbitrary bounded nonlinear activation functions.
\newblock \emph{IEEE Transactions on Neural Networks}, 1998.

\bibitem[Hwang(2023)]{hwang23}
Geonho Hwang.
\newblock Minimum width for deep, narrow mlp: A diffeomorphism and the whitney
  embedding theorem approach.
\newblock \emph{arXiv preprint arXiv:2308.15873}, 2023.

\bibitem[Johnson(2019)]{johnson18}
Jesse Johnson.
\newblock Deep, skinny neural networks are not universal approximators.
\newblock In \emph{International Conference on Learning Representations
  (ICLR)}, 2019.

\bibitem[Kidger and Lyons(2020)]{kidger20}
Patrick Kidger and Terry Lyons.
\newblock {Universal approximation with deep narrow networks}.
\newblock In \emph{Conference on Learning Theory (COLT)}, 2020.

\bibitem[Leshno et~al.(1993)Leshno, Lin, Pinkus, and
  Schocken]{journals/nn/Leshno93}
Moshe Leshno, Vladimir~Ya. Lin, Allan Pinkus, and Shimon Schocken.
\newblock Multilayer feedforward networks with a nonpolynomial activation
  function can approximate any function.
\newblock \emph{Neural Networks}, 6\penalty0 (6):\penalty0 861--867, 1993.

\bibitem[Li et~al.(2022)Li, Lin, and Shen]{li22}
Qianxiao Li, Ting Lin, and Zuowei Shen.
\newblock Deep learning via dynamical systems: An approximation perspective.
\newblock \emph{Journal of the European Mathematical Society}, 2022.

\bibitem[Lu(2021)]{lu21}
Zhou Lu.
\newblock A note on the representation power of ghhs.
\newblock \emph{arXiv preprint arXiv:2101.11286}, 2021.

\bibitem[Lu et~al.(2017)Lu, Pu, Wang, Hu, and Wang]{Lu17}
Zhou Lu, Hongming Pu, Feicheng Wang, Zhiqiang Hu, and Liwei Wang.
\newblock The expressive power of neural networks: A view from the width.
\newblock In \emph{Annual Conference on Neural Information Processing Systems
  (NeurIPS)}, 2017.

\bibitem[Munkres(2000)]{munkres}
James~R. Munkres.
\newblock \emph{Topology}.
\newblock Prentice Hall, Inc., 2000.

\bibitem[Park et~al.(2021{\natexlab{a}})Park, Lee, Yun, and Shin]{park21b}
Sejun Park, Jaeho Lee, Chulhee Yun, and Jinwoo Shin.
\newblock Provable memorization via deep neural networks using sub-linear
  parameters.
\newblock In \emph{Conference on Learning Theory (COLT)}, 2021{\natexlab{a}}.

\bibitem[Park et~al.(2021{\natexlab{b}})Park, Yun, Lee, and Shin]{park21}
Sejun Park, Chulhee Yun, Jaeho Lee, and Jinwoo Shin.
\newblock {Minimum width for universal approximation}.
\newblock In \emph{International Conference on Learning Representations
  (ICLR)}, 2021{\natexlab{b}}.

\bibitem[Pinkus(1999)]{pinkus99}
Allan Pinkus.
\newblock Approximation theory of the mlp model in neural networks.
\newblock \emph{Acta Numerica}, 8:\penalty0 143 -- 195, 1999.

\bibitem[Rudin(1987)]{rudin}
Walter Rudin.
\newblock \emph{Real and Complex Analysis}.
\newblock McGraw-Hill, Inc., 1987.

\bibitem[Song et~al.(2023)Song, Hwang, Lee, and Kang]{song23}
Changhoon Song, Geonho Hwang, Junho Lee, and Myungjoo Kang.
\newblock Minimal width for universal property of deep rnn.
\newblock \emph{Journal of Machine Learning Research}, 2023.

\bibitem[Telgarsky(2016)]{telgarsky16}
Matus Telgarsky.
\newblock Benefits of depth in neural networks.
\newblock In \emph{Conference on Learning Theory (COLT)}, 2016.

\bibitem[Vardi et~al.(2022)Vardi, Yehudai, and Shamir]{vardi22}
Gal Vardi, Gilad Yehudai, and Ohad Shamir.
\newblock On the optimal memorization power of {ReLU} neural networks.
\newblock In \emph{Conference on Learning Theory (COLT)}, 2022.

\bibitem[Vershynin(2020)]{vershynin20}
Roman Vershynin.
\newblock Memory capacity of neural networks with threshold and rectified
  linear unit activations.
\newblock \emph{SIAM Journal on Mathematics of Data Science}, 2020.

\bibitem[Wang and Qu(2022)]{wang22}
Ming-Xi Wang and Yang Qu.
\newblock Approximation capabilities of neural networks on unbounded domains.
\newblock \emph{Neural Networks}, 2022.

\bibitem[Yarotsky(2018)]{yarotsky18}
Dmitry Yarotsky.
\newblock Optimal approximation of continuous functions by very deep {ReLU}
  networks.
\newblock In \emph{Conference on Learning Theory (COLT)}, 2018.

\bibitem[Yun et~al.(2019)Yun, Sra, and Jadbabaie]{yun19}
Chulhee Yun, Suvrit Sra, and Ali Jadbabaie.
\newblock Small {ReLU} networks are powerful memorizers: a tight analysis of
  memorization capacity.
\newblock In \emph{Annual Conference on Neural Information Processing Systems
  (NeurIPS)}, 2019.

\end{thebibliography}
\bibliographystyle{plainnat}

\newpage
\appendix

\section{Definition of activation functions}\label{sec:activation}

In this section, we introduce the definitions of activation functions that we mainly focus on.

\begin{itemize}[leftmargin=15pt]

\item ReLU ($\relu$):
\begin{align*}
    \relu(x) = 
        \begin{cases}
        x~&\text{if}~x > 0\\
        0~&\text{if}~x \le 0
        \end{cases}.
\end{align*}

\item Softplus ($\softplus$): for $\alpha>0$,
\begin{align*}
    \softplus(x;\alpha) = \frac{1}{\alpha}\log(1+\exp(\alpha x)).
\end{align*}

\item LeakyReLU (Leaky-$\relu$): for $\alpha\in(0,1)$,
\begin{align*}
    \text{Leaky-}\relu(x;\alpha) = 
    \begin{cases}
    x~&\text{if}~x > 0\\
    \alpha x~&\text{if}~x \le 0
    \end{cases}.
\end{align*}

\item Exponential Linear Unit ($\elu$): for $\alpha>0$,
\begin{align*}
    \elu(x;\alpha) = 
    \begin{cases}
    x~&\text{if}~x > 0\\
    \alpha \left( \exp(x)-1 \right)~&\text{if}~x \le 0
    \end{cases}.
\end{align*}

\item Continuously differentiable Exponential Linear Unit ($\celu$): for $\alpha>0$,
\begin{align*}
    \celu(x;\alpha) = 
    \begin{cases}
    x~&\text{if}~x > 0\\
    \alpha \left( \exp(x/\alpha)-1 \right)~&\text{if}~x \le 0
    \end{cases}.
\end{align*}

\item Scaled Exponential Linear Unit ($\selu$): for $\lambda>1$ and $\alpha>0$,
\begin{align*}
    \selu(x;\lambda,\alpha) = \lambda \times
    \begin{cases}
    x~&\text{if}~x > 0\\
    \alpha \left( \exp(x)-1 \right)~&\text{if}~x \le 0
    \end{cases}.
\end{align*}

\item Gaussian Error Linear Unit ($\gelu$):
\begin{align*}
    \gelu(x) = x \times \Phi (x)
\end{align*}
where $\Phi(x)$ is the cumulative distribution function of the standard normal distribution.

\item Sigmoid Linear Unit ($\silu$):
\begin{align*}
    \silu(x) = x \times \sigmoid(x)
\end{align*}
where $\sigmoid(x) = 1/\left(1+\exp(-x)\right)$ is the sigmoid activation function.

\item Mish ($\mish$):
\begin{align*}
    \mish(x) = x \times \tanhh ( \softplus(x;1) )
\end{align*}
where $\tanhh(x) = (\exp(x)-\exp(-x)) / (\exp(x)+\exp(-x))$. %

\end{itemize} 

\newpage
\section{Proof of upper bound in \cref{thm:lp-ub}}\label{sec:pfthm:lp-ub}

\subsection{Additional notations}\label{sec:add-notation}
Prior to delving into the proof of \cref{thm:lp-ub}, we first introduce some notations that will be frequently employed in the subsequent sections.
Given a set $\mathcal S\subset\mathbb R^n$, $\interior(\mathcal S)$ denotes the interior of $\mathcal S$, and $\bd(\mathcal S)$ denotes the boundary of $\mathcal S$.
For $(a,b) \in \mathbb R^n \times \mathbb R$, $\mathcal H(a,b) \defeq \{ x \in \mathbb R^n : a^\top x + b = 0\}$ denotes the hyperplane parameterized by $a$ and $b$.
Likewise, we use $\mathcal H^+(a,b) \defeq \{ x \in \mathbb R^n : a^\top x + b \ge 0\}$ and $\mathcal H^-(a,b) \defeq \{ x \in \mathbb R^n : a^\top x + b \le 0\}$ for denoting corresponding upper half-space and lower half-space, respectively.

\subsection{Our choices of $\alpha,\beta,\gamma$}\label{sec:choice-abc}
We use a small enough $\alpha>0$ so that $\omega_{p,2,f^*}(\alpha)\le\varepsilon/2^{1+1/p}$, $\beta=\varepsilon^p/(2d_y)$, and $\gamma=\varepsilon/2^{1+1/p}$.
Here, $\omega_{p,2,f^*}$ denotes the modulus of continuity of $f^*$ in the $p$-norm and $2$-norm: $\|f^*(x)-f^*(x')\|_p\le\omega_{p,2,f^*}(\|x-x'\|_2)$ for all $x,x'\in[0,1]^{d_x}$. We note that such $\omega_{p,2,f^*}$ is well-defined on $[0,1]^{d_x}$ since $f^*$ is uniformly continuous on $[0,1]^{d_x}$ (continuous function on a compact set).
Namely, such $\alpha$ always exists for all $\varepsilon>0$.
Then, we have
\begin{align*}
\|f^*-f\|_p^p&=\int_{[0,1]^{d_x}}\|f^*(x)-f(x)\|_p^pd\mu_{d_x}\\
&=\int_{[0,1]^{d_x}\setminus\bigcup_{i=1}^k\mathcal T_i}\|f^*(x)-f(x)\|_p^pd\mu_{d_x}+\int_{\bigcup_{i=1}^k\mathcal T_i}\|f^*(x)-f(x)\|_p^pd\mu_{d_x}\\
&\le d_y\times\mu_{d_x}\left([0,1]^{d_x}\setminus\bigcup_{i=1}^k\mathcal T_i\right)+\sum_{i=1}^k\int_{\mathcal T_i}(\|f(x)-f^*(x)\|_p)^pd\mu_{d_x}\\
&\le d_y\times\beta+\sum_{i=1}^k\int_{\mathcal T_i}(\|f^*(x)- f^*(z_i)\|_p+\|f(x)- f^*(z_i)\|_p)^pd\mu_{d_x}\\
&\le d_y\times\beta+\sum_{i=1}^k\int_{\mathcal T_i}(\omega_{p,2,f^*}(\alpha)+\gamma)^pd\mu_{d_x}\\
&\le d_y\times\beta+(\omega_{p,2,f^*}(\alpha)+\gamma)^p\le\varepsilon^p
\end{align*}
where $z_i\in\mathcal T_i$ for all $i$.

This leads us to the statement of \cref{lem:ub-lp}.

\subsection{Proof of \cref{lem:encoder}}\label{sec:pflem:encoder}

We prove  \cref{lem:encoder} in this section.
To this end, we introduce the following lemma.
The proof of \cref{lem:poly-cut} is presented in \cref{sec:pflem:poly-cut}. Here, $\mathcal H^+(a,b)$ is defined in \cref{sec:add-notation}

\begin{lemma}\label{lem:poly-cut}
Let $\mathcal K_0=[0,1]^{n}$.
For any $\delta>0$, there exist $m\in\mathbb N$ and $(a_1,b_1),\dots,(a_m,b_m)\in\mathbb R^n\times\mathbb R$ such that for $\mathcal K_i = \mathcal K_0\cap\ \mathcal H_i^+$ where $ \mathcal H_i^+ = \bigcap_{j=1}^i \mathcal H^+(a_j,b_j)$ for all $i\in[m]$, 
\begin{align*}
\diam(\mathcal K_0\setminus\mathcal K_1),\diam(\mathcal K_{1}\setminus\mathcal K_2),\dots,\diam(\mathcal K_{m-1}\setminus\mathcal K_m)\le\delta\quad\text{and}\quad\mathcal K_m=\emptyset.
\end{align*}
\end{lemma}

\cref{lem:poly-cut} ensures the existence of the partition $\{ \mathcal S_1, \dots, \mathcal S_k \}$ of the domain $[0,1]^{d_x}$ such that the diameter of each $\mathcal S_i$ is upper bounded by given $\alpha >0$ where each $\mathcal S_i$ can be represented as
\begin{align*}
\mathcal S_i=[0,1]^{d_x}\cap\bigg(\bigcap_{j=1}^{i-1} \mathcal H^+(a_j,b_j)\bigg)\cap\big(\mathcal H^+(a_i,b_i)\big)^c
\end{align*} 
for some $(a_1,b_1),\dots,(a_i,b_i) \in \mathbb R^n \times \mathbb R$.

To show the existence of an approximation of a partition $\{ \mathcal T_1, \dots, \mathcal T_k \}$ and a $\relu$ network of width $\max\{d_x,2\}$ that maps $\mathcal T_1, \dots, \mathcal T_k$ to some distinct points, we introduce the following lemma.
The proof of \cref{lem:basic-encoder} is presented in \cref{sec:pflem:basic-encoder}.

\begin{lemma}\label{lem:basic-encoder}
Let $n\in\mathbb N$, $\mathcal P\subset\mathbb R^n$ be a convex polytope and $u_1,\dots,u_k\in\mathbb R^n\setminus\mathcal P$ be distinct points. Let $\mathcal H^+\subset\mathbb R^n$ be a closed half-space such that $\mathcal P\cap\mathcal H^+\ne\emptyset$ and $\mathcal P\setminus\mathcal H^+\ne\emptyset$.
{Let $m = \max\{n,2\}$.}
Then for any $\delta\in(0,1)$, there exists a $\relu$ network 
{$f:\mathbb R^n\to\mathbb R^m$}
of width 
{$m$}
satisfying the following properties:
\begin{itemize}[leftmargin=15pt]
    \item $f(x)=x$ for all $x\in\mathcal P\cap\mathcal H^+$,
    \item there exist distinct $v_1,\dots,v_{k}\in
    {\mathbb R^m}
    \setminus(\mathcal P\cap\mathcal H^+)$ such that $f(u_i)=v_i$ for all $i\in[k]$, and
    \item there exist $\mathcal S\subset\mathcal P\setminus\mathcal H^+$ and $v_{k+1}\in
    {\mathbb R^m}
    \setminus((\mathcal P\cap\mathcal H^+)\cup\{v_1,\dots,v_k\})$ such that $\mu_n(\mathcal S)\ge\delta\cdot\mu_n(\mathcal P\setminus\mathcal H^+)$ and $f(\mathcal S)=\{v_{k+1}\}$.
\end{itemize}
\end{lemma}

\cref{lem:basic-encoder} indicates that for each $\mathcal S_i$, there is a $\relu$ network $g_i$ of width $\max\{d_x,2\}$ such that $g_i$
(i) preserves the points in $\mathcal S_{i+1} \cup \cdots \cup \mathcal S_k$,
(ii) transfers the $u_1, \dots, u_{i-1}$ (containing ``information'' that we want to preserve) to some distinct points $v_1, \dots, v_{i-1}$ not contained in $\mathcal S_{i+1} \cup \cdots \cup \mathcal S_k$, and
(iii) embeds the most part of $\mathcal S_i$ (i.e. $\mathcal T_{i}$) to the point $v_{i}$ distinct to existing points $v_1, \dots, v_{i-1}$ and also not contained in $\mathcal S_{i+1} \cup \cdots \cup \mathcal S_k$. Hence, by repeatedly applying \cref{lem:basic-encoder}, we can find $\{ \mathcal T_1,\dots,\mathcal T_k \}$ and a $\relu$ network $g$ of width $\max\{d_x,2\}$ such that $\mu_n(\mathcal T_i)\ge\mu_n(\mathcal S_i)-\beta/k$ for all $i\in[k]$ and the network maps each $\mathcal T_i$ to 
{$u_i\in\mathbb R^{m}$}
for some distinct $u_1,\dots,u_k$
{where $m=\max\{d_x,2\}$ }.

Lastly, we introduce the following lemma, which demonstrates the existence of a projection map that maps finite distinct points to some distinct scalar values.
The proof of \cref{lem:distinct-innerprod} is presented in \cref{sec:pflem:distinct-innerprod}.

\begin{lemma}\label{lem:distinct-innerprod}
For any {$n \ge 2$ and} distinct $v_1,\dots,v_k\in\mathbb R^n$, there exists $a\in\mathbb R^n$ such that
$a^\top v_1,\dots,a^\top v_k$ are also distinct.
\end{lemma}

By \cref{lem:distinct-innerprod}, there exists an affine map $h$ that maps $u_1,\dots,u_k$ to distinct scalar values.
Then, choosing $f=h\circ g$ completes the proof of \cref{lem:encoder}.

\subsection{Proof of \cref{lem:decoder0}}\label{sec:pflem:decoder0}

In this section, we prove \cref{lem:decoder0} by explicitly constructing the target $f$.
As aforementioned, the proof of \cref{lem:decoder0} is a corollary of Lemma~9 and Lemma~10 in \citep{park21}.

Before describing our proof details, we first introduce some functions introduced in \citep{park21}.
A quantization function $q_n:[0,1] \to \mathcal C_n$ for $n\in \mathbb N$ and $\mathcal C_n\defeq\{0,2^{-n}, 2\times2^{-n}, 3\times2^{-n},\dots,1-2^{-n}\}$ is defined as
\begin{align*}
    q_n(x) = \max \{c \in \mathcal C_n : c \le x \},
\end{align*}
an encoder $\mathrm{encode}_K : \mathbb R^{d_x} \to \mathcal C_{d_x K}$ for some $K \in \mathbb N$ is defined as
\begin{align*}
    \mathrm{encode}_{K}(x) = \sum\nolimits_{i=1}^{d_{x}} q_K(x_i) \times 2^{-(i-1)K},
\end{align*}
and a decoder $\mathrm{decode}_M : \mathcal C_{d_y M} \to \mathcal C_M^{d_y}$ is defined as 
\begin{align*}
    \mathrm{decode}_M(c) = \hat{x} \quad \text{where} \quad \{ \hat{x} \} = \mathrm{encode}_M^{-1}(c) \cap \mathcal C_M^{d_y}.
\end{align*}

Namely, $\mathrm{encode}_K$ quantizes every coordinate of input up to $K$-bits and then concatenates whole coordinates into a one-dimensional scalar value.
And, $\mathrm{decode}_M$ decodes one-dimensional codewords to $d_y$-dimensional codewords.

Now, we introduce the following lemmas presented in \citep{park21}.

\begin{lemma}[Lemma~9 in \citep{park21}]\label{lem:encoder2}
For any $m\in\mathbb N$, $0\le\alpha_1<\alpha_2<\cdots<\alpha_m\le1$, and $\beta_1,\dots,\beta_m\in\mathbb R$, there exists a $\relu$ network $f:[0,1]\to\mathbb R$ of width $2$ such that 
\begin{align*}
f(\alpha_i)=\beta_i\quad\text{for all}~i\in[m].
\end{align*}
\end{lemma}

\begin{lemma}[Lemma~10 in \citep{park21}]\label{lem:decoder}
For any $d_y,M\in\mathbb{N}$, there exists a $\relu$ network $f:\mathbb R\to\mathbb R^{d_y}$ of width $d_y$ such that for any $c\in\mathcal C_{d_yM}$, 
\begin{align*}
    f(c)=(b_1,\dots,b_{d_y})
\end{align*}
where $b_1,\dots,b_{d_y}\in\mathcal C_M$ satisfying $c=\sum_{i=1}^{d_y}b_i\times2^{-(i-1)M}$.
Furthermore, it holds that $f(\mathbb R)\subset[0,1]^{d_y}$.
\end{lemma}

Using \cref{lem:encoder2}, we first exactly construct a $\relu$ network $g$ of width $2$ which maps each codeword $c_i$ to the corresponding encoded target vector $\mathrm{encode}_{M}(v_i) \in \mathcal C_{d_yM}$ for some $M \in \mathbb N$; we will assign an explicit value to $M$ later.
Next, by \cref{lem:decoder}, we explicitly construct a $\relu$ network $h$ of width $d_y$ which maps each $\mathrm{encode}_{M}(v_i)$ to being the $d_y$-dimensional quantized target vector $v_i^\dagger$ where $v_i^\dagger = (v_{i,1}^\dagger, \dots, v_{i, d_y}^\dagger)$ and $v_{i,j}^\dagger=q_M(v_{i,j})$ for each $j \in [d_y]$.

Let $f$ be the composition of $\relu$ networks $g$ and $h$.
That is, $f$ is a $\relu$ network of width $\max \{ d_y, 2\}$.
From the construction of $g$ and $h$, the error between $f(c_i)  =v_i^\dagger$ and $v_i$ is only incurred from the quantization process.
Hence, choosing sufficiently large $M \in \mathbb N$ such that $d_y^{1/p}\times2^{-M} \le \gamma$ completes the proof of \cref{lem:decoder0}.

For the sake of completeness, we provide proofs of \cref{lem:encoder2} and \cref{lem:decoder}, which are from \citep{park21}.
\begin{proof}[Proof of \cref{lem:encoder2}]
Consider the following piecewise linear function $f^*:[0,1]\to\mathbb R$ with $m+1$ pieces which satisfies the statement of \cref{lem:encoder2}:
\begin{align*}
f(x)=\begin{cases}
\beta_1~&\text{if}~x \in [0, \alpha_1)\\
\beta_i + \dfrac{\beta_{i+1} - \beta_i}{\alpha_{i+1} - \alpha_i}(x - \alpha_i) ~&\text{if}~x \in [\alpha_{i}, \alpha_{i+1})~\text{for some}~ i \in[m-1]\\
\beta_m~&\text{if}~x \in [\alpha_m, 1]\\
\end{cases}.
\end{align*}

From \cref{lemma:piecewise-relu}, we can construct a $\relu$ network $f:[0,1]\to\mathbb R$ of width $2$ satisfying $f^*(x) = f(x)$ for all $x \in [0,1].$ This completes the proof of \cref{lem:encoder2}.
\end{proof}

\begin{lemma}[Lemma~14 in \citep{park21}]\label{lemma:piecewise-relu}
For any compact interval $\mathcal I \subset \mathbb R$, for any continuous piecewise linear function $f^* : \mathcal I \subset \mathbb R$ with $P$ linear pieces, there exists a $\relu$ network $f$ of width $2$ such that $f(x) = f^*(x)$ for all $x \in \mathcal I$.
\end{lemma}

\begin{proof}[Proof of \cref{lemma:piecewise-relu}]
Suppose that $f^*$ is linear on $P$ pieces $[\min \mathcal I, x_1), [x_1, x_2), \dots, [x_{P-1}, \max \mathcal I]$ and defined as
\begin{align*}
f(x)=\begin{cases}
a_1 \times x + b_1~&\text{if}~x \in [\min \mathcal I, x_1)\\
a_2 \times x + b_2~&\text{if}~x \in [x_1, x_2)\\
&\vdots \\
a_P \times x + b_P~&\text{if}~x \in [x_{P-1}, \max \mathcal I]\\
\end{cases}
\end{align*}
for some $a_i, b_i \in \mathbb R$ satisfying $a_i \times x_i + b_i = a_{i+1} \times x_i + b_{i+1}$.
Without loss of generality, we assume that $\min \mathcal I = 0$.

Now, we prove that for any $P \ge 1$, there exists a $\relu$ network $f : \mathcal I \to \mathbb R^2$ of width $2$ such that $f(x)_1 = \relu(x-x_{P-1})$ and $f(x)_2 = f^*(x)$.
Then, the $\relu$ network $f(x)_2$ completes the proof.
We use the mathematical induction for $P$ to prove the existence of corresponding $f$. 
When $P=1$, choosing $f(x)_1 = \relu(x)$ and $f(x)_2 = a_1 \times \relu(x) + b_1$ satisfies the desired property.
Here, consider $P>1$. 
Then, from the induction hypothesis, there exists a $\relu$ network $g$ of width $2$ such that
\begin{align*}
&g(x)_1 = \relu(x - x_{P-2})\\
&g(x)_2=
\begin{cases}
a_1 \times x + b_1~&\text{if}~x \in [\min \mathcal I, x_1)\\
a_2 \times x + b_2~&\text{if}~x \in [x_1, x_2)\\
&\vdots \\
a_{P-1} \times x + b_{P-1}~&\text{if}~x \in [x_{P-2}, \max \mathcal I]\\
\end{cases}
\end{align*}
Then, the following construction of $f$ completes the proof of the mathematical induction:
\begin{align*}
    f(x) &= h_2 \circ \phi \circ h_1 \circ g(x)\\
    h_1(x,z) &= (x - x_{P-1} + x_{P-2}, z-K)\\
    \phi(x,z) &= (\relu(x), \relu(z))\\
    h_2(x,z) &= (x, K + z + (a_{P} - a_{P-1})\times x)
\end{align*}
where $K = \min_i \min_{x\in\mathcal I} \{ a_i \times x + b_i\}$. Hence, this completes the proof of \cref{lemma:piecewise-relu}.
\end{proof}

\begin{proof}[Proof of \cref{lem:decoder}]
We first introduce the following lemma.
\begin{lemma}[Lemma~15 in \citep{park21}]\label{lem:mini_decoder}
For any $M \in \mathbb N$, for any $\delta > 0$, there exists a $\relu$ network of $f : \mathbb R \to \mathbb R^2$ of width $2$ such that for all $x \in [0,1]\setminus\mathcal D_{M,\delta}$,
\begin{align}\label{eq:dec:0}
    f(x) = (y_1(x), y_2(x)), \quad \text{where} \quad y_1(x) = q_M(x), \quad y_2(x) = 2^M \times (x-q_M(x)),
\end{align}
and $\mathcal D_{M,\delta} = \bigcup_{i=1}^{2^M-1}(i \times 2^{-M} - \delta , i \times 2^{-M})$. Furthermore, it holds that
\begin{align}\label{eq:dec:1}
    f(\mathbb R) \subset [0, 1 - 2^{-M}] \times [0,1].
\end{align}
\end{lemma}

Fix some $\delta < 2^{- d_y M}$. Then, \cref{lem:mini_decoder} indicates that there exists a $\relu$ network $g$ of width $2$ satisfying (\ref{eq:dec:0}) on $\mathcal C_{d_y M}$ and (\ref{eq:dec:1}) since $C_{d_y M} \subset [0,1] \setminus \mathcal D_{M,\delta}$.
Such $g$ enables us to extract the first $M$ bits of the binary representation of $c \in \mathcal C_{d_y M}$: $g(c)_1$ is the first coordinate of $\mathrm{decode}_M(c)$ while $g(c)_2 \in C_{d_{y-1}M}$ contains remaining information about other coordinates of $\mathrm{decode}_M(c)$. 
Therefore, if we iteratively apply $g$ to the second output of the previous composition of $g$ and pass through all first outputs of the previous compositions of $g$, then we finally recover whole coordinates of $\mathrm{decode}_M(c)$ within $d_y - 1$ compositions of $g$. 
Our construction of $f$ is such iterative $d_y - 1$ compositions of $g$ which can be implemented by a $\relu$ network of width $d_y$. 
Moreover, (\ref{eq:dec:1}) in \cref{lem:mini_decoder} allows us to achieve $f(\mathbb R) \subset [0,1]^{d_y}$. This completes the proof of \cref{lem:decoder}.
\end{proof}

\begin{proof}[Proof of \cref{lem:mini_decoder}]
First of all, we clip the input to be in $[0,1]$ using the following $\relu$ network of width $1$.
\begin{align*}
    \min \{ \max\{x,0\} ,1\} = 1 - \relu(1-\relu(x))
\end{align*}
Then, we apply $g_\ell : [0,1] \to [0,1]^2$ defined as
\begin{align*}
    g_\ell(x)_1 &= x\\
    g_\ell(x)_2 &= 
    \begin{cases}
0~&\text{if}~x \in [0, 2^{-M}-\delta]\\
\delta^{-1}2^{-M} \times (x-2^{-M} + \delta)~&\text{if}~x \in (2^{-M}-\delta, 2^{-M})\\
2^{-M}~&\text{if}~x \in [2^{-M},2 \times 2^{-M}-\delta]\\
\delta^{-1}2^{-M} \times (x-2\times 2^{-M} + \delta) + 2^{-M}~&\text{if}~x \in (2\times 2^{-M}-\delta, 2\times 2^{-M})\\
&\vdots \\
(\ell-1) \times 2^{-M}~&\text{if}~x \in [(\ell -1) \times 2^{-M}, 1]
\end{cases}.
\end{align*}
From the definition of $g_\ell$, one can observe that $g_{2^M}(x)_2 = q_M(x)$ for $x \in [0,1] \setminus \mathcal D_{M,\delta}$.
Therefore, once we implement $g_{2^M}(x)$ using a $\relu$ network $g$ of width $2$, and then constructing $f$ as
\begin{align*}
    f(x) &= (  g(z)_2, 2^M \times (g(z)_1 - g(z)_2) )\\
    z &= \min \{ \max\{x,0\} ,1\} = 1 - \relu(1-\relu(x))
\end{align*}
completes the proof of \cref{lem:mini_decoder}. 
Now, we construct a $\relu$ network $g$ of width $2$ which implements $g_{2^M}$. One can observe that $g_1(x)_2 = 0$ and
\begin{align*}
    g_{\ell+1}(x)_2 = \min \left\{  \ell \times 2^{-M}, \max\{ \delta^{-1}2^{-M} \times (x-\ell \times 2^{-M} + \delta) + (\ell - 1)\times 2^{-M}  ,g_\ell (x) \} \right\}
\end{align*}
for all $x \in [0,1]$.
To this end, we introduce the following definition and lemma.
\begin{definition}[Definition~1 in \citep{hanin17}]
$f : \mathbb R^{d_x} \to \mathbb R^{d_y}$ is a max-min string of length $L \ge 1$ if there exist affine transformations $h_1, \dots, h_L$ such that
\begin{align*}
    h(x) = \tau_{L-1}(h_L(x), \tau_{L-2}( h_{L-1}(x),\dots,\tau_2(h_3(x), \tau_1(h_2(x),h_1(x)  )) \dots    ),
\end{align*}
where each $\tau_\ell$ is either a coordinate-wise $\max \{\cdot,\cdot \}$ or $\min \{\cdot,\cdot \}$.
\end{definition}
\begin{lemma}[Proposition~2 in \citep{hanin17}]
For any max-min string $f:\mathbb R^{d_x} \to \mathbb R^{d_y}$ of length $L$, for any compact $\mathcal K \subset \mathbb R^{d_x}$, there exists a $\relu$ network $g : \mathbb R^{d_x} \to \mathbb R^{d_x} \times \mathbb R^{d_y}$ of $L$ layers and width $d_x + d_y$ such that for all $x \in \mathcal K$,
\begin{align*}
    g(x) = (y_1(x), y_2(x)), \quad \text{where} \quad y_1(x) = x \quad \text{and} \quad y_2(x) = f(x).
\end{align*} 
\end{lemma}
Notably, $g_{2^M}(x)$ is a max-min string so that there exists a $\relu$ network $g$ of width $2$ satisfying $g(x)_2 = g_{2^M}(x) = q_M(x)$ for all $x \in [0,1]\setminus\mathcal D_{M,\delta}$. This completes the proof of \cref{lem:mini_decoder}.
\end{proof}

\subsection{Proof of \cref{lem:tool1}}\label{sec:pflem:tool1}
Let $a_1=a$, $\{a_2,\dots,a_n\}$ be a basis of the hyperplane $\{x\in\mathbb R^n:c^\top x=0\}$, and let $A=\begin{bmatrix}a_1,\dots,a_n\end{bmatrix}^\top\in\mathbb R^{n\times n}$. Since $a^\top c\ne0$, $A$ is invertible.
Choose 
\begin{align*}
    K=-1\times\min_{i\in\{2,\dots,n\}}\inf_{x\in\mathcal K}a_i^\top x
\end{align*}
and $v=(b,K,\dots,K)\in\mathbb R^n$.
Then we claim that choosing 
\begin{align*}
f(x)=A^{-1}(\relu(Ax+v)-v)
\end{align*}
completes the proof.
Here, we use $\relu$ for multi-dimensional input by applying $\relu$ element-wise.
From our choice of $A,v,$ and $K$, if $a_1^\top x+b\ge0$, then $f(x)=x$.
Suppose that $a_1^\top x+b<0$. 
In this case, from our choice of $K$, 
the second to the last coordinates of $\relu(Ax+v)$ are identical to that of $Ax+v$. 
Namely, we have
\begin{align}
f(x)&=A^{-1}\big((Ax+v-(a_1^\top x+b)e_1)-v\big)=x-(a_1^\top x+b)A^{-1}e_1\label{eq:pflem:tool1}
\end{align}
where $e_1=(1,0,\dots,0)\in\mathbb R^n$.
Since $a_1^\top c>0$ and $a_i^\top c=0$ for all $i\in\{2,\dots,n\}$, the first column of $A^{-1}$ (i.e., $A^{-1}e_1$) must be $c/a_1^\top c=c/a^\top c$.
Therefore, by \cref{eq:pflem:tool1}, it holds that
\begin{align*}
f(x)=x-\frac{a^\top x+b}{a^\top c}\times c.
\end{align*}
This completes the proof of \cref{lem:tool1}.

\subsection{Proof of \cref{lem:poly-cut}}\label{sec:pflem:poly-cut}
The statement of \cref{lem:poly-cut} directly follows from repeatedly applying \cref{claim:poly-cut}. 
Specifically, we iteratively construct $(a_1,b_1),\dots$ and corresponding $\mathcal K_1,\dots$ using \cref{claim:poly-cut} so that $\diam(\mathcal K_0\setminus\mathcal K_1),\dots\le\delta$ until $\diam(\mathcal K_r)\le\delta$ for some $r\in\mathbb N$. 
Then, we choose $(a_{r+1},b_{r+1})$ such that 
\begin{align*}
\mathcal K_{r+1}=\mathcal K_r\cap\mathcal H^+(a_{r+1},b_{r+1})=\emptyset.
\end{align*}
Setting $m=r+1$ completes the proof.
\begin{claim}\label{claim:poly-cut}
For any $\delta>0$ and bounded set $\mathcal R_0\subset\mathbb R^n$ with $\diam(\mathcal R_0)\le D$ for some $D>0$, there exist $k\in\mathbb N$ and $(c_1,d_1),\dots,(c_k,d_k)\subset\mathbb R^n\times\mathbb R$ such that
\begin{itemize}[leftmargin=15pt]
\item $\diam(\mathcal R_0\setminus\mathcal R_1),\dots,\diam(\mathcal R_{k-1}\setminus\mathcal R_{k})\le\delta$ and
\item $\diam(\mathcal R_k)\le\max\{D-\delta^2/(4D),0\}$
\end{itemize}
where $\mathcal R_i=\mathcal R_0\cap\big(\bigcap_{j=1}^i\mathcal H^+(c_j,d_j)\big)$ for all $i\in[k]$.
\end{claim}
\begin{proof}
Without loss of generality, we assume that $\mathcal R_0\subset\mathcal B_0$ where $\mathcal B_0$ denotes the $n$-dimensional closed $\ell_2$-ball of radius $D/2$, centered at the origin. 
In addition, we assume that $D>\delta$; otherwise, the statement of \cref{claim:poly-cut} trivially follows.
Let $r=\sqrt{(D^2-\delta^2)/4}$, $\gamma=D/2-r$, $u_x=-x/\|x\|_2$ for each $x\in\bd(\mathcal B_0)$, and $\mathcal S^+_x=\mathcal H^+(u_x,r)$. Then, we have 
\begin{align}
\diam(\mathcal B_0\setminus\mathcal S_x^+)\le\delta.\label{eq:poly-cut}
\end{align}
Let $\mathcal B_1$ be an $n$-dimensional closed $\ell_2$-balls of radius 
$D/2 - \gamma/{2}>r$, centered at the origin, i.e., $\mathcal B_1 \subset \mathcal B_0$.
Since $\mathcal B_0 \setminus \interior (\mathcal B_1)$ is compact and $\{\mathbb R^n\setminus\mathcal S_x^+:x\in\bd(\mathcal B_0)\}$ is an open cover of $\mathcal B_0 \setminus \interior (\mathcal B_1)$, there exists a finite set $\mathcal I\subset\bd(\mathcal B_0)$ such that 
\begin{align}
\mathcal B_{0} \setminus \interior (\mathcal B_1)\subset\bigcup_{x\in\mathcal I}(\mathbb R^n\setminus\mathcal S_x^+).\label{eq:poly-cut2}
\end{align} 

Let $k=|\mathcal I|$, $\mathcal I=\{x_1,\dots,x_k\}$,  $c_i=u_{x_i}$, and $d_i=\gamma$ for all $i\in[k]$. Then by \cref{eq:poly-cut}, it holds that 
\begin{align*}
\diam(\mathcal R_{i-1}\setminus\mathcal R_i)=\diam(\mathcal R_{i-1}\setminus\mathcal S_{x_i}^+)\le\diam(\mathcal B_0\setminus\mathcal S_{x_i}^+)\le\delta
\end{align*}
for all $i\in[k]$.
Furthermore, by \cref{eq:poly-cut2} and $\mathcal R_0\subset\mathcal B_0$, we have
\begin{align*}
\mathcal R_k=\mathcal R_0\setminus\Big(\bigcup_{x\in\mathcal I}(\mathbb R^n\setminus\mathcal S_x^+)\Big)\subset \mathcal B_0\setminus\big(\mathcal B_{0} \setminus \interior (\mathcal B_1)\big)\subset\mathcal B_1.
\end{align*}
This implies that
\begin{align*}
\diam(\mathcal R_k)\le\diam(\mathcal B_1)=D-\gamma=\frac{D}2+\frac{\sqrt{D^2-\delta^2}}2\le D-\frac{\delta^2}{4D}
\end{align*}
where the last inequality follows from the concavity of the square root: $\sqrt{a-b}\le \sqrt{a}-b/(2\sqrt{a})$ for $a\ge b>0$.
This completes the proof of \cref{lem:poly-cut}.
\end{proof}

\subsection{Proof of \cref{lem:basic-encoder}}\label{sec:pflem:basic-encoder}
We first introduce the following lemmas.
\begin{lemma}\label{lem:move-pt}
Let {$n \ge 2$}, $\mathcal P\subset\mathbb R^n$ be a convex polytope, $z_1,\dots,z_k\in\mathbb R^n\setminus\mathcal P$ be distinct points, and $\mathcal H^+\subset\mathbb R^n$ be a closed half-space. 
Suppose that $z_1,\dots,z_{l-1}\in\mathcal H^+$ and $z_{l},\dots,z_k\notin\mathcal H^+$ for some $l\in[k]$.
Then, there exists a $\relu$ network $f:\mathbb R^n\to\mathbb R^n$ of width $n$ satisfying the following:
\begin{itemize}[leftmargin=15pt]
    \item $f(x)=x$ for all $x\in\mathcal P$,
    \item $f(z_1),\dots,f(z_k)$ are distinct, $f(z_i)\notin\mathcal P$ for all $i\in[k]$, and $f(z_1),\dots,f(z_l)\in\mathcal H^+$.
\end{itemize}
\end{lemma}
\begin{lemma}\label{lem:poly-to-point}
Let {$n \ge 2$}, $\mathcal P\subset\mathbb R^n$ be a convex polytope, $\mathcal H^+\subset\mathbb R^n$ be a closed half-space such that $\mathcal P\cap\mathcal H^+\ne\emptyset$ and $\mathcal P\setminus\mathcal H^+\ne\emptyset$, and $v_1,\dots,v_k\in\mathcal H^+\setminus\mathcal P$ be distinct points. 
Then for any $\delta\in(0,1)$, there exists a $\relu$ network $f:\mathbb R^n\to\mathbb R^n$ of width $n$ satisfying the following:
\begin{itemize}[leftmargin=15pt]
    \item $f(x)=x$ for all $x\in(\mathcal P\cap\mathcal H^+)\cup\{v_1,\dots,v_k\}$,
    \item there exist $\mathcal S\subset\mathcal P\setminus\mathcal H^+$ and $v_{k+1}\in\mathbb R^n\setminus((\mathcal P\cap\mathcal H^+)\cup\{v_1,\dots,v_k\})$ such that $\mu_n(\mathcal S)\ge\delta\cdot\mu_n(\mathcal P\setminus\mathcal H^+)$ and $f(\mathcal S)=\{v_{k+1}\}$.
\end{itemize}
\end{lemma}

Consider the case $n\ge2$.
By repeatedly applying \cref{lem:move-pt}, one can construct a $\relu$ network $g_1$ of width $n$ such that
$g_1(x)=x$ for all $x\in\mathcal P$, $g_1(u_1),\dots,g_1(u_k)\in\mathcal H^+\setminus\mathcal P$, and $g_1(u_1),\dots,g_1(u_k)$ are distinct.
Let $v_i=g_1(u_i)$ for all $i\in[k]$.
Next, we apply \cref{lem:poly-to-point} to $\mathcal P$, $v_1,\dots,v_k$, and $\mathcal H^+$. Then, one can construct a $\relu$ network $g_2$ such that $g_2(x)=x$ for all $x\in(\mathcal P\cap\mathcal H^+)\cup\{v_1,\dots,v_k\}$ and $g_2(\mathcal S)=\{v_{k+1}\}$ for some $v_{k+1}\notin(\mathcal P\cap\mathcal H^+)\cup\{v_1,\dots,v_k\}$ and $\mathcal S\subset\mathcal P\setminus\mathcal H^+$ with $\mu_n(\mathcal S)\ge\delta\cdot\mu_n(\mathcal P\setminus\mathcal H^+)$.
Choosing $f=g_2\circ g_1$ completes the proof.

For the case $n=1$, we can not directly apply \cref{lem:move-pt} and \cref{lem:poly-to-point}.
Nonetheless, by exploiting the inclusion map $\iota : x \mapsto (x,0)$ for $x \in \mathbb R$, we can also yield the $\relu$ network $f=g_2\circ g_1 \circ \iota$ of width $2$, {which satisfies the statement of \cref{lem:basic-encoder}}.

\begin{proof}[Proof of \cref{lem:move-pt}]
Let $a_1\in\mathbb R^n\setminus\{0\}$ and $b_1\in\mathbb R$ such that $\mathcal H^+=\{x\in\mathbb R^n:a_1^\top x+b_1\ge0\}$.
Since $\mathcal P$ and $\{z_l\}$ are disjoint closed convex sets, by the hyperplane separation theorem \citep{boyd04}, there exist $a_2\in\mathbb R^n\setminus\{0\}$ and $b_2\in\mathbb R$ such that 
\begin{align}
a_2^\top x+b_2>0~~\text{for all}~~x\in\mathcal P~~\text{and}~~a_2^\top z_l+b_2<0.\label{eq:lem:move-pt1}
\end{align}
Without loss of generality, we assume that $a_2$ is not in the span of $a_1$. Otherwise, we can consider a slightly perturbed version of $a_2$ that is not in the span of $a_1$ and achieves \eqref{eq:lem:move-pt1};
such perturbation always exists since $\mathcal P$ and $\{z_l\}$ are bounded.

Now, for $c\in\mathbb R^n\setminus\{0\}$ such that $a_2^\top c>0$, we apply \cref{lem:tool1} to construct a $\relu$ network $f_c$ of width $n$ of the following form:
\begin{align*}
f_c(x)=\begin{cases}
x~&\text{if}~a_2^\top x+b_2\ge 0\\
x-\frac{a_2^\top x+b_2}{a_2^\top c}\times c~&\text{if}~a_2^\top x+b_2<0
\end{cases}.
\end{align*}
Then, from our choice of $a_2$, $b_2$, and $f_c$, (i) $f_c(x)=x$ for all $x\in\mathcal P$.
Furthermore, consider $c\in\mathcal S$ where 
\begin{align*}
    \mathcal S\defeq\left\{x\in\mathbb R^n\setminus\{0\}:\frac{a_1^\top z_l+b_1}{a_2^\top z_l+b_2}\le\frac{a_1^\top x}{a_2^\top x},a_1^\top x>0,a_2^\top x>0\right\}.
\end{align*}
Then, we have
\begin{align*}
a_1^\top f_c(z_l)+b_1=a_1^\top z_l-\frac{a_2^\top z_l+b_2}{a_2^\top c}(a_1^\top c)+b_1\ge0
\end{align*}
where the inequality is from the definition of $\mathcal S$ and $a_2^\top z_l+b_2<0$.
In addition, for any $x\in\mathbb R^n$, it holds that $a_1^\top f_c(x)\ge a_1^\top x$. These inequalities imply that (ii) $f_c(z_1),\dots,f_c(z_l)\in\mathcal H^+$.
Furthermore, we have (iii) $f_c(z_1),\dots,f_c(z_k)\notin\mathcal P$: if $a_2^\top z_i+b_2\ge0$, then $f_c(z_i)=z_i\notin\mathcal P=f_c(\mathcal P)$; otherwise, then $a_2^\top f_c(z_i)+b_2=0<a_2^\top x+b_2=a_2^\top f_c(x)+b_2$ for all $x\in\mathcal P$.
Here, one can observe that $\mu_n(\mathcal S)>0$ (i.e., $\mathcal S$ is non-empty) since $a_1$ and $a_2$ are linearly independent. 

Lastly, we show that there exists $c\in\mathcal S$ such that $f_c(z_1),\dots,f_c(z_k)$ are distinct. From the definition of $f_c$ and since $z_i\ne z_j$ for all $i\ne j$, $f_c(z_i)=f_c(z_j)$ only if $z_i-z_j$ and $c$ are linearly dependent.
However, the set of vectors that is in the span of $z_i-z_j$ has zero measure with respect to $\mu_n$; however, $\mu_n(\mathcal S)>0$. This implies that there exists $c\in\mathcal S$ such that $f_c(z_1),\dots,f_c(z_k)$ are distinct.
By (i)--(iii), choosing such $c\in\mathcal S$ and $f=f_c$ completes the proof.
\end{proof}
\begin{proof}[Proof of \cref{lem:poly-to-point}]
Without loss of generality, we assume that the normal vector of the boundary of $\mathcal H^+$ is $(1,0,\dots,0)$, i.e.,
we consider
\begin{align*}
\mathcal H^+_1=\mathcal H^+=\{(x_1,\dots,x_n)\in\mathbb R^n:x_1+b_1\ge0\}
\end{align*}
for some $b_1\in\mathbb R$.
Let $\mathcal T=(\mathcal P\cap\mathcal H_1^+)\cup\{v_1,\dots,v_k\}$,
\begin{align*}
b_i=1-1\times\min_{(x_1,\dots,x_n)\in\mathcal T}x_i~~\text{and}~~\mathcal H_i^+=\{(x_1,\dots,x_n)\in\mathbb R^n:x_i+b_i\ge0\}
\end{align*}
for all $i\in[n]\setminus\{1\}$. We note that $b_i$ is well-defined since $\mathcal T$ is compact. Furthermore, from the definition of $\mathcal H_i^+$, it holds that $\mathcal T\subset\bigcap_{i=1}^n\mathcal H_i^+$.

For $\gamma>0$, let $\mathcal K_\gamma= {\mathcal P \cap} \{(x_1,\dots,x_n)\in\mathbb R^n:x_1+b_1+\gamma\le0\}$, i.e., $\mathcal K_\gamma\subset\mathcal P\setminus\mathcal H_1^+$.
Due to {the continuity of Lebesgue measure}, there exists a small enough $\gamma^*>0$ such that $\mu_n(\mathcal K_{\gamma^*})\ge\delta\cdot\mu_n(\mathcal P\setminus\mathcal H_1^+)$.
Here, we choose $\mathcal S=\mathcal K_{\gamma^*}$.
In detail, let $\gamma_n = 1/n$ and consider corresponding $\mathcal K_{\gamma_n}$ and an indicator function $\indc_{\mathcal K_{\gamma_n}}$.
Then, our choice $\indc_{\mathcal K_{\gamma_n}}$ is monotonically increasing on $\mathcal P\setminus \mathcal H_1^+$ and converges almost everywhere on $\mathcal P\setminus \mathcal H_1^+$ to $\indc_{\mathcal K_0}$ = $\indc_{\mathcal P \setminus \mathcal H_1^+}$.
Then, by Lebesgue's Monotone Convergence Theorem \citep{rudin},
$\mu_n(K_{\gamma_n})$ converges to $\mu_n(\mathcal P \setminus \mathcal H_1^+)$.
Namely, we can choose $N\in\mathbb N$ such that $\mu_n(\mathcal K_{\gamma_{N}}) \ge \delta \cdot \mu_n(\mathcal P \setminus \mathcal H_1^+)$.

Let $\beta_i=\sup_{(x_1,\dots,x_n)\in\mathcal K_{\gamma^*}}|x_i|$ for all $i\in[n]\setminus\{1\}$;  each $\beta_i$ is finite since $\mathcal K_{\gamma^*}$ is bounded. 
Now, using \cref{lem:tool1}, we construct a $\relu$ network $f_1$ of width $n$ of the following form: for $x=(x_1,\dots,x_n)\in\mathbb R^n$ and $c_1=\big(1,-\max\{b_2+\beta_2+1,1\}/\gamma^*,\dots,-\max\{b_n+\beta_n+1,1\}/\gamma^*\big)$,
\begin{align*}
f_1(x)=\begin{cases}
x~&\text{if}~x_1+b_1\ge 0\\
x-(x_1+b_1) c_1~&\text{if}~x_1+b_1<0
\end{cases}.
\end{align*}
From the definition of $f_1$, we have $f_1(x)=x\in\bigcap_{i=1}^n\mathcal H_i^+$ for all $x\in\mathcal T$.
In addition, from the definition of $\mathcal K_{\gamma^*}$, we have $x_1+b_1+\gamma^*\le0$ for all $(x_1,\dots,x_n)\in\mathcal K_{\gamma^*}$. This implies that $f_1(x)_1=-b_1$ and 
\begin{align*}
    f_1(x)_i&=x_i-\frac{-(x_1+b_1)\max\{b_i+\beta_i+1,1\}}{\gamma^*}\\
    &\le x_i-{\max\{b_i+\beta_i+1,1\}}\\
    &\le\beta_i-{\max\{b_i+\beta_i+1,1\}}\\
    &\le -b_i-1
\end{align*} for all $i\in[n]\setminus\{1\}$ and for all $x\in\mathcal K_{\gamma^*}$.
Here, the first inequality is from $-(x_1+b_1)\ge\gamma^*\ge0$ and $\max\{b_i+\beta_i+1,1\}>0$. We use $\beta_i\ge x_i$ for the second inequality.
This implies that
$f_1(\mathcal K_{\gamma^*})\cap(\bigcup_{i=2}^n\mathcal H_i^+)=\emptyset$.

Now, we construct $\relu$ networks $f_2,\dots,f_n$ of width $n$ using \cref{lem:tool1} as follows: for $x=(x_1,\dots,x_n)\in\mathbb R^n$ and $i\in[n]\setminus\{1\}$,
\begin{align*}
f_i(x)=\begin{cases}
x~&\text{if}~x_i+b_i\ge 0\\
(x_1,\dots,x_{i-1},-b_i,x_{i+1},\dots,x_n)~&\text{if}~x_i+b_i<0
\end{cases}.
\end{align*}
Here, one can observe that $f_i(x)=x\in\bigcap_{i=1}^n\mathcal H_i^+$ for all $x\in\mathcal T$ and $i\in[n]\setminus\{1\}$.
Furthermore, since $f_1(x)_1=-b_1$ for all $x\in\mathcal K_{\gamma^*}$ and
$f_1(\mathcal K_{\gamma^*})\cap(\bigcup_{i=2}^n\mathcal H_i^+)=\emptyset$, we have 
\begin{align*}
f_n\circ\cdots\circ f_2\circ f_1(\mathcal K_{\gamma^*})=(-b_1,-b_2,\dots,-b_n).
\end{align*}
Since $(-b_1,\dots,-b_n)\notin\mathcal T$ (e.g., $\min_{(x_1,\dots,x_n)\in\mathcal T}x_2=1-b_2>-b_2$), choosing $f=f_n\circ\cdots\circ f_1$ and $v_{k+1}=(-b_1,\dots,-b_n)$ completes the proof.
\end{proof}

\subsection{Proof of \cref{lem:distinct-innerprod}}\label{sec:pflem:distinct-innerprod}

Let $\mathcal H_{ij}=\{x\in\mathbb R^n:x^\top(v_i-v_j)=0\}$ for all $i<j$.
Then,
$a^\top v_1,\dots,a^\top v_k$ are distinct if and only if $a\notin\bigcup_{i<j}\mathcal H_{ij}$.
However, since $\mu_n(\mathcal H_{ij})=0$, we have $\mu_n(\bigcup_{i<j}\mathcal H_{ij})=0$, i.e., there exists $a\in\mathbb R^n\setminus(\bigcup_{i<j}\mathcal H_{ij})$.
This completes the proof.

\newpage
\section{Proof of lower bounds in \cref{thm:lp-ub} and \cref{cor:general-lp}}\label{sec:pf:general-lb}

We first introduce the following lemmas.

\begin{lemma}[Lemma~1 in \citep{cai23}]\label{lem:dxdy}
For any activation function, networks of width $\max\{d_x,d_y\}$ are not dense in both $L^p(\mathcal K, \mathbb R^{d_y})$ and $C(\mathcal K, \mathbb R^{d_y})$.
\end{lemma}

\begin{lemma}\label{lem:counter-ex-monotone}
For any continuous monotone $\psi$, $\psi$ networks of width $1$ are not dense in $L^p([0,1],\mathbb R)$.
\end{lemma}

From \cref{lem:dxdy} and \cref{lem:counter-ex-monotone}, the proof of the lower bounds in \cref{thm:lp-ub} and \cref{cor:general-lp} directly follow.

\begin{proof}[Proof of \cref{lem:counter-ex-monotone}]
In this proof, we show that for any continuous monotone $\psi$, there exists a $L^p$ measurable function $f^* : [0,1] \to \mathbb R$ that cannot be approximated by any $\psi$ network of width $1$, say $f$, within $1/6$ error measured by $L^p$ norm.
Here, by the H{\"o}lder's inequality, it suffices to show that $ \| f - f^*\|_1 \ge 1/6$.

Since any compositions of continuous monotone functions are continuous monotone, without loss of generality, we assume that a $\varphi$ network $f$ is a continuous and monotonically increasing function.

Consider $f^* : [0,1] \to \mathbb R$ defined as
\begin{align*}
    f^*(x) = 
    \begin{cases}
    0~&\text{if}~x \in [0,1/3] \cup [2/3,1]\\
    1~&\text{if}~x \in (1/3,2/3)
    \end{cases},
\end{align*}
and let $f(2/3) = c$ for some $c \in \mathbb R$.
Then if $c \le 0$, 
\begin{align*}
    \int_{[0,1]} | f - f^* | dx \ge \int_{[1/3,2/3]} |  f - f^* | dx \ge {\int_{[1/3,2/3]} |f^*| dx = \frac{1}{3}}.
\end{align*}

Likewise, if $c \ge 1$,
\begin{align*}
    \int_{[0,1]} | f - f^* | dx \ge \int_{[2/3,1]} | f - f^* | dx \ge {\int_{[2/3,1]} | 1 - f^* | = \frac{1}{3}}.
\end{align*}

Furthermore, if $c \in (0,1)$,
\begin{align*}
    \int_{[0,1]} | f - f^* | dx &\ge \int_{[1/3,2/3]} | f - f^* | dx + \int_{[2/3,1]} | f - f^* | dx \\
    &\ge \int_{[1/3,2/3]} | c - f^* | dx + \int_{[2/3,1]} | c - f^* | dx = \frac{1}{3}(1-c) + \frac{1}{3}c = \frac{1}{3}.
\end{align*}

Hence, for any continuous monotone $\varphi$ network can not approximate $f^*$ within $1/6$ error, which completes the proof.
\end{proof}

\newpage
\section{Proof of upper bound in \cref{cor:general-lp}}\label{sec:pfcor:general-lp-ub}

\subsection{Additional notations}

Throughout this section, we use $\relul$ for the set of $\relu$-like activation functions of our interest, defined as
\begin{align*}
\relul\defeq\{\softplus,\text{Leaky-}\relu,\elu,\celu,\selu,\gelu,\silu,\mish\}.
\end{align*}
For $n \in \mathbb N$ and $f : \mathbb R \to \mathbb R$, we denote $f^n(x) \defeq f \circ \cdots \circ f$ the $n$-th iterate of the function $f$.
For any continuous function $f$ on a compact domain $\mathcal K \subset \mathbb R^{n}$, $\omega_{p,f}$ denotes the modulus of continuity of $f$ in the $p$-norm: $\|f(x)-f(x')\|_p\le\omega_{p,f}(\|x-x'\|_p)$ for all $x,x'\in\mathcal K$. We note that such $\omega_{p,f}$ is well-defined on a compact domain since $f$ is uniformly continuous on a compact domain.

\subsection{Proof of upper bound in \cref{cor:general-lp}}\label{sec:pfcor:general-lp}

In this proof, we show that for any $\varepsilon > 0$, $\varphi\in\relul$, and $\relu$ network $f$ of width $w$, there exists a $\varphi$ network $g$ with the same width $w$ such that
\begin{align*}
    \| f - g \|_p \le \varepsilon.
\end{align*}
Then, combining the above bound and the upper bound $w_{\min}\le\{ d_x, d_y, 2\}$ in \cref{thm:lp-ub} completes the proof of the upper bound in \cref{cor:general-lp}.

To this end, we first introduce the following lemma.
The proof of \cref{lem:relu-like-converge} is presented in \cref{sec:pflem:relu-like-converge}.

\begin{lemma}\label{lem:relu-like-converge}
For any given compact set $\mathcal K \subset \mathbb R$ and activation function $\varphi \in \relul$, there exists a sequence $\{h_n\}$ of $\varphi$ networks of width $1$ such that it uniformly converges to $\relu$ on $\mathcal K$.
\end{lemma}

\cref{lem:relu-like-converge} states that for any given compact $\mathcal K \subset \mathbb R$, for each $\relu$-like activation function $\varphi$, and some fixed $\delta >0$, there exist a sequence $\{ h_n\}$ of $\varphi$ networks of width $1$ and $N \in \mathbb N$ such that $\| \relu(x) - h_N(x) \|_\infty \le \delta$ on all $n\ge N$ and $x \in \mathcal K$; we will assign an explicit value to $\delta$ later.
We note that it suffices to show uniform convergence on an arbitrary compact set $\mathcal K$ since functions of our interests are defined on compact domains.

Here, we denote a $\relu$ network $f$ as below, recalling (\ref{eq:def-nn}):
\begin{align*}
    f = t_L \circ \phi_{L-1} \circ \cdots \circ t_2 \circ \phi_1 \circ t_1 
\end{align*}
where $L\in \mathbb N$ is the number of layers,  $t_\ell : \mathbb R^{d_{\ell-1}} \to \mathbb R^{d_\ell}$ is an affine transformation, and $\phi_\ell(x_1,\dots,x_{d_\ell}) = \left( \relu(x_1), \dots, \relu(x_{d_\ell}) \right)$ 
for all $\ell\in[L]$.
And, for each $\relu$-like activation function $\varphi$, we choose a $\varphi$ network $g$ via applying the same affine maps $t_1, \dots, t_L$ and $h_n$, which is satisfying \cref{lem:relu-like-converge}, such that
\begin{align*}
    g = t_L \circ \rho_{L-1} \circ \cdots \circ t_2 \circ \rho_1 \circ t_1 
\end{align*}
where $\rho_\ell(x_1,\dots,x_{d_\ell}) = \left( h_n(x_1), \dots, h_n(x_{d_\ell}) \right)$ for all $\ell\in[L]$.
We further denote $f_\ell$ and $g_\ell$ by the first $\ell-1$ layers of $f$ and $g$ with the subsequent affine layer $t_\ell$, respectively:
\begin{align*}
    f_{\ell} =  t_{\ell} \circ \phi_{\ell-1} \circ \cdots \circ \phi_1 \circ t_1 \quad\text{and} \quad g_{\ell} = t_{\ell} \circ \rho_{\ell-1} \circ \cdots \circ \rho_1 \circ t_1
\end{align*}
Then, for each $\ell \in [L] \setminus \{ 1 \}$, we have
\begin{align}
\|f_\ell-g_\ell\|_p&= \| t_\ell \circ \phi_{\ell-1} \circ f_{\ell-1} - t_\ell \circ \rho_{\ell-1} \circ g_{\ell-1} \|_p \nonumber \\
&\le \omega_{t_\ell,p} \left( \| \phi_{\ell-1} \circ f_{\ell-1} - \rho_{\ell-1} \circ g_{\ell-1} \|_p \right) \nonumber \\
&\le \omega_{t_\ell,p} \left(  \| \phi_{\ell-1} \circ f_{\ell-1} - \phi_{\ell-1} \circ g_{\ell-1} \|_p +  \| \phi_{\ell-1} \circ g_{\ell-1} - \rho_{\ell-1} \circ g_{\ell-1} \|_p \right) \nonumber \\
&\le \omega_{t_\ell,p} \Bigg( \| \phi_{\ell-1} \circ f_{\ell-1} - \phi_{\ell-1} \circ g_{\ell-1} \|_p \nonumber\\
&\qquad \qquad \qquad + \left( \int_{[0,1]^{d_x}} \| \phi_{\ell-1} \circ g_{\ell-1} (x) - \rho_{\ell-1} \circ g_{\ell-1} (x) \|_p^p d \mu_{d_x}   \right)^{1/p} \Bigg)\nonumber\\
&= \omega_{t_\ell,p} \Bigg( \| \phi_{\ell-1} \circ f_{\ell-1} - \phi_{\ell-1} \circ g_{\ell-1} \|_p \nonumber\\
&\qquad \qquad \qquad + \Bigg( \int_{[0,1]^{d_x}} \sum_{i=1}^{d_{\ell-1}} \Big( \relu( g_{\ell-1}(x)_i ) - h_n(g_{\ell -1}(x)_i) \Big)^p d \mu_{d_x}   \Bigg)^{1/p} \Bigg)\nonumber
\end{align}
We note that $\omega_{t_\ell,p}$ is well-defined on $[0,1]^{d_x}$ since $t_\ell$ is uniformly continuous on $[0,1]^{d_x}$.

For each $i \in [d_{\ell-1}]$, by \cref{lem:relu-like-converge}, there exists $N_i \in \mathbb N$ such that
\begin{align*}
    \| \relu( g_{\ell -1}(x)_i ) - h_n( g_{\ell -1}(x)_i ) \|_\infty \le \delta
\end{align*}
for all $n \ge N_i$ and $x \in [0,1]^{d_x}$. 
Moreover, from the definition of $\phi_{\ell-1}$, we have
\begin{align*}
    \| \phi_{\ell-1} \circ f_{\ell-1} - \phi_{\ell-1} \circ g_{\ell-1} \|_p \le \omega_{\phi_{\ell-1},p}(\| f_{\ell-1} -g_{\ell-1}\|_p) \le \| f_{\ell-1} -g_{\ell-1}\|_p.
\end{align*}
Therefore, for $n \ge \max \{ N_1, \dots, N_{d_{\ell -1}} \}$, we have
\begin{align}
\|f_\ell-g_\ell\|_p
&\le \omega_{t_\ell,p} \Bigg( \| \phi_{\ell-1} \circ f_{\ell-1} - \phi_{\ell-1} \circ g_{\ell-1} \|_p \nonumber\\
&\qquad \qquad \qquad + \Bigg( \int_{[0,1]^{d_x}} \sum_{i=1}^{d_{\ell-1}} \Big( \relu( g_{\ell-1}(x)_i ) - h_n(g_{\ell -1}(x)_i) \Big)^p d \mu_{d_x}   \Bigg)^{1/p} \Bigg)\nonumber\\
& \le  \omega_{t_\ell,p} \left( \| f_{\ell-1} - g_{\ell-1} \|_p +  \delta \times d_{\ell-1}^{1/p}\right), \label{ineq:lp-norm}
\end{align}
with $\| f_1 - g_1 \|_p = \| t_1 - t_1 \|_p = 0.$

Consequently, by iteratively applying (\ref{ineq:lp-norm}), we get
\begin{align}
\|f-g\|_p & \le  \omega_{t_L,p} \left( \| f_{L-1} - g_{L-1} \|_p +  \delta \times d_{L-1}^{1/p} \right) \nonumber \\
&\le \omega_{t_L,p} \left( \omega_{t_{L-1},p} \left( \| f_{L-2} - g_{L-2} \|_p +  \delta \times d_{L-2}^{1/p} \right) +  \delta \times d_{L-1}^{1/p}\right) \nonumber \\
&~~\vdots \nonumber \\
&\le\omega_{t_L,p} \left( \omega_{t_{L-1},p} \left( \cdots \omega_{t_2,p}\left( \| f_{1} - g_{1} \|_p +  \delta \times d_1^{1/p}\right) \cdots +  \delta \times d_{L-2}^{1/p} \right) +  \delta \times d_{L-1}^{1/p}\right) \nonumber \\
&=\omega_{t_L,p} \left( \omega_{t_{L-1},p} \left( \cdots \omega_{t_2,p}\left( \delta \times d_1^{1/p}\right) \cdots +  \delta \times d_{L-2}^{1/p} \right) +  \delta \times d_{L-1}^{1/p}\right) .\label{ineq:genearl-lp-ub}
\end{align}
Thus, we can bound the right-hand side (\ref{ineq:genearl-lp-ub}) of the above inequality within any $\varepsilon >0 $, by choosing sufficiently small $\delta > 0$.
Hence, it completes the proof of the upper bound in \cref{cor:general-lp}.

\subsection{Proof of \cref{lem:relu-like-converge}}\label{sec:pflem:relu-like-converge}

In this section, we explicitly construct a sequence of $\varphi$ network $\{ h_n \}$ satisfying \cref{lem:relu-like-converge}.
Namely, we show that for any $\varepsilon > 0$, for any compact set $\mathcal K \subset \mathbb R$, there exists $N\in\mathbb N$ such that $| h_n(x) - \relu(x) | \le \varepsilon$ for all $n \ge N$ and $x \in \mathcal K$.
Without loss of generality, we assume that $\mathcal K = [-m, M]$ for some $m,M > 0$.

$1$. $\varphi = \softplus$:

In this case, we claim that $h_n(x) = (t_2 \circ \varphi \circ t_1) (x)$ completes the proof for $\softplus$, where $t_1(x)=nx$ and $t_2(x) = x/n$. If $x \ge 0$, by the Mean Value Theorem, we have
\begin{align*}
    | h_n(x) - \relu(x) | &= \frac{1}{\beta n} \log(1+\exp(\beta n x)) - \frac{1}{\beta n} \log(\exp(\beta n x))\\
    &\le \frac{1}{\beta n}.
\end{align*}
Otherwise, if $x < 0$,
\begin{align*}
    | h_n(x) - \relu(x) | = \frac{1}{\beta n} \log(1+\exp(\beta n x)) < \frac{1}{\beta n} \log(2)
\end{align*}
since $\softplus$ is strictly increasing. 
Hence, choosing sufficiently large $N \in \mathbb N$ such that $1/\beta N \le \varepsilon$, which completes the proof for $\softplus$.

$2$. $\varphi = \text{\rm Leaky-}\relu$:

In this case, we claim that $h_n(x) = \varphi^n(x)$ completes the proof for $\text{\rm Leaky-}\relu$.
Here, from the definition of $\text{\rm Leaky-}\relu$, we only consider for $x < 0$. Then,
\begin{align*}
    | h_n(x) - \relu(x) | = \alpha^n | x | \le \alpha^n \times m
\end{align*}
We note that $\alpha \in (0,1)$.
Hence, choosing sufficiently large $N \in \mathbb N$ such that $\alpha^N \times m \le \varepsilon$, which completes the proof for $\text{Leaky-}\relu$.

$3$. $\varphi = \elu$:

Similar to the case $\varphi = \text{\rm leaky-}\relu$, we claim that $h_n(x) = \varphi^n(x)$ completes the proof for $\elu$ and only consider for $x < 0$ from the definition of $\elu$.
Since $\elu$ is bounded below by $-\alpha$, strictly increasing, and $\elu(0) = 0$, we have
\begin{align*}
|h_1(x)| < \alpha &\Rightarrow |h_2(x)| < \alpha\left(1-\exp(-\alpha)\right) < \alpha   \\
&\Rightarrow |h_3(x)|< \alpha(1 - \exp(-\alpha\left(1-\exp(-\alpha)\right)))< \alpha\left(1-\exp(-\alpha)\right) 
\end{align*}
We note that the upper bound of the sequence $\{|h_n|\}$ is strictly decreasing and its infimum is equal to $0$.
Thus, by the monotone convergence theorem, there exists $N \in \mathbb N$ such that $| h_n(x) - \relu(x)|=| h_n(x) | \le \varepsilon$ for all $n \ge N$.
Hence, this completes the proof for $\elu$.

$4$. $\varphi = \celu$:

Since $\celu$ is a smooth variant of $\elu$, the proof technique for $\celu$ is the same as $\elu$.
Consider $h_n(x) = \varphi^n(x)$.
Then, for $x<0$, we have
\begin{align*}
|h_1(x)| < \alpha &\Rightarrow |h_2(x)| < \alpha(1 - 1/ e) < \alpha   \\
&\Rightarrow |h_3(x)|< \alpha(1-\exp(1/e-1)) < \alpha(1 - 1/ e)
\end{align*}
Likewise, we note that the upper bound of the sequence $\{|h_n|\}$ is strictly decreasing and its infimum is equal to $0$.
Thus, by the monotone convergence theorem, there exists $N \in \mathbb N$ such that $| h_n(x) - \relu(x)|=| h_n(x) | \le \varepsilon$ for all $n \ge N$.
Hence, this completes the proof for $\celu$.

$5$. $\varphi = \selu$:

From the definition of $\elu$ and $\selu$, we can represent $\selu$ as $\lambda \times \elu$.
Hence, from the proof for $\elu$, $h_n(x) = (t_{\lambda}\circ \varphi)^n(x)$ completes the proof for $\selu$, where $t_\lambda(x) = x/\lambda$.

$6$. $\varphi = \gelu$:

In this case, we claim that $h_n(x) = (t_2 \circ \varphi \circ t_1) (x)$ completes the proof for $\gelu$, where $t_1(x)=nx$ and $t_2(x) = x/n$.
If $x \ge 0$, we have
\begin{align*}
    | \varphi_n(x) - \relu(x) | = x (1 - \Phi(nx)) \le x\exp(-n^2x^2/2) \le \frac{1}{n\sqrt{e}}
\end{align*}
since $P(Z \ge t) \le \exp(-t^2/2)$ for all $t \ge 0$ where $Z \sim \mathcal N(0,1)$.
We note that the last inequality is derived from its derivative.
Otherwise, if $x < 0$,
\begin{align*}
    | \varphi_n(x) - \relu(x) | &= x \Phi(nx) \le x\exp(-n^2x^2/2) \le \frac{1}{n\sqrt{e}}
\end{align*}
since $P(Z \le t) \le \exp(-t^2/2)$ for all $t \le 0$ where $Z \sim \mathcal N(0,1)$. Hence, choosing sufficiently large $N \in \mathbb N$ such that $1/N \le \varepsilon$, which completes the proof for $\gelu$.

$7$. $\varphi = \silu$:

In this case, we claim that $h_n(x) = (t_2 \circ \varphi \circ t_1) (x)$ completes the proof for $\silu$, where $t_1(x)=nx$ and $t_2(x) = x/n$. 
If $x \ge 0$, we have
\begin{align*}
    | \varphi_n(x) - \relu(x) | &= x \left( 1 - \frac{1}{1+\exp(-nx)}  \right) = \frac{x \exp(-nx)}{1+\exp(-nx)}\\
    &< x \exp(-nx) \le \frac{1}{en}.
\end{align*}
Note that the last inequality is derived from its derivative. 
Otherwise, if $x < 0$,
\begin{align*}
    | \varphi_n(x) - \relu(x) | = \frac{-x}{1+\exp(-nx)}  \le \frac{1}{n}.
\end{align*}
since $1-x \le \exp(-x)$ for all $x \in \mathbb R$.
Hence, choosing sufficiently large $N \in \mathbb N$ such that $1/N \le \varepsilon$, which completes the proof for $\silu$.

$8$. $\varphi = \mish$:

In this case, we claim that $h_n(x) = (t_2 \circ \varphi \circ t_1) (x)$ completes the proof for $\mish$, where $t_1(x)=nx$ and $t_2(x) = x/n$.
If $x \ge 0$, we have
\begin{align*}
    | \varphi_n(x) - \relu(x) | &= x \left( 1 - \frac{(1 + \exp(nx))^2 -1}{(1 + \exp(nx))^2 +1}  \right)\\ &= \frac{2x}{(1 + \exp(nx))^2 +1} \le \frac{x}{1 + \exp(nx)} \le \frac{1}{n}.
\end{align*}
since $1+x \le \exp(x)$ for all $x\in\mathbb R$.
Otherwise, if $x < 0$,
\begin{align*}
    | \varphi_n(x) - \relu(x) | &= -x \times \frac{(1 + \exp(nx))^2 -1}{(1 + \exp(nx))^2 +1} \\ &\le -x \times \frac{(1 + \exp(nx))^2 -1}{2 + \exp(nx) }\\ &= -x\exp(nx) \le \frac{1}{en}.
\end{align*}
Note that the last inequality is derived from its derivative.
Hence, choosing sufficiently large $N \in \mathbb N$ such that $1/N\le \varepsilon$, which completes the proof for $\mish$.

\newpage
\section{Proof of \cref{lem:inj-approx}}\label{sec:pflem:inj-approx}

In this proof, we show that for any $\varepsilon>0$, $\sigma$ that can be uniformly approximated by some sequence of continuous injection, say $\varphi_n$, and $\sigma$ network $f$ of width $w$, there exists $\varphi_n$ network $g$ of the same width $w$ such that 
\begin{align*}
    \| f-g \|_\infty \le \varepsilon.
\end{align*}
From the assumption, there exists a sequence of continuous injection $\{ \varphi_n \}$ that uniformly converges to $\sigma$ on $\mathbb R$.
Namely, for any $\delta>0$, there exists $N \in \mathbb N$ such that $\| \varphi_n(x) - \sigma(x) \|_\infty \le \delta$ for all $n \ge N$ and $x \in \mathbb R$; we will assign an explicit value to $\delta$ later.

Now, we denote a $\sigma$ network $f$ as below, recalling (\ref{eq:def-nn}):
\begin{align*}
    f = t_L \circ \phi_{L-1} \circ \cdots \circ t_2 \circ \phi_1 \circ t_1 
\end{align*}
where $L\in \mathbb N$ is the number of layers,  $t_\ell : \mathbb R^{d_{\ell-1}} \to \mathbb R^{d_\ell}$ is an affine transformation, and $\phi_\ell(x_1,\dots,x_{d_\ell}) = \left( \sigma(x_1), \dots, \sigma(x_{d_\ell}) \right)$ 
for all $\ell\in[L]$.
And, we choose a $\varphi_n$ network 
such that
\begin{align*}
    g = t_L \circ \rho_{L-1} \circ \cdots \circ t_2 \circ \rho_1 \circ t_1 
\end{align*}
where $\rho_\ell(x_1,\dots,x_{d_\ell}) = \left( \varphi_n(x_1), \dots, \varphi_n(x_{d_\ell}) \right)$ for all $\ell\in[L]$.
We further denote $f_\ell$ and $g_\ell$ by the first $\ell-1$ layers of $f$ and $g$ with the subsequent affine layer $t_\ell$, respectively:
\begin{align*}
    f_{\ell} =  t_{\ell} \circ \phi_{\ell-1} \circ \cdots \circ \phi_1 \circ t_1 \quad\text{and} \quad g_{\ell} = t_{\ell} \circ \rho_{\ell-1} \circ \cdots \circ \rho_1 \circ t_1
\end{align*}

Then, for each $\ell \in [L] \setminus \{ 1 \}$, we have
\begin{align*}
\|f_\ell-g_\ell\|_\infty&= \| t_\ell \circ \phi_{\ell-1} \circ f_{\ell-1} - t_\ell \circ \rho_{\ell-1} \circ g_{\ell-1} \|_\infty \nonumber \\
&\le \omega_{t_\ell,\infty} \left( \| \phi_{\ell-1} \circ f_{\ell-1} - \rho_{\ell-1} \circ g_{\ell-1} \|_\infty \right) \nonumber \\
&\le \omega_{t_\ell,\infty} \left(  \| \phi_{\ell-1} \circ f_{\ell-1} - \phi_{\ell-1} \circ g_{\ell-1} \|_\infty +  \| \phi_{\ell-1} \circ g_{\ell-1} - \rho_{\ell-1} \circ g_{\ell-1} \|_\infty \right) \nonumber \\
&\le \omega_{t_\ell,\infty} \Big( \| \phi_{\ell-1} \circ f_{\ell-1} - \phi_{\ell-1} \circ g_{\ell-1} \|_\infty \nonumber\\
&\qquad \qquad \qquad + \sup_{x\in [0,1]^{d_x}} \| \phi_{\ell-1} \circ f_{\ell-1}(x) - \rho_{\ell-1} \circ g_{\ell-1}(x)  \|_\infty \Big)\nonumber\\
&=\omega_{t_\ell,\infty} \Big( \| \phi_{\ell-1} \circ f_{\ell-1} - \phi_{\ell-1} \circ g_{\ell-1} \|_\infty \nonumber\\
&\qquad \qquad \qquad + \sup_{x\in [0,1]^{d_x}} \max_{i \in [d_{\ell-1}]} \left\{ \sigma( g_{\ell -1}(x)_i ) - \varphi_n( g_{\ell -1}(x)_i )  \right\} \Big).\nonumber
\end{align*}
We note that $\omega_{t_\ell,\infty}$ is well-defined on $[0,1]^{d_x}$ since $t_\ell$ is uniformly continuous on $[0,1]^{d_x}$.

Since $\sigma$ is uniformly approximated by $\varphi_n$, for each $i \in [d_{\ell-1}]$, there exists $N_i \in \mathbb N$ such that
\begin{align*}
    \| \sigma( g_{\ell -1}(x)_i ) - \varphi_n( g_{\ell -1}(x)_i ) \|_\infty \le \delta
\end{align*}
for all $n \ge N_i$ and $x \in [0,1]^{d_x}$. 
Therefore, for $n \ge \max \{ N_1, \dots, N_{d_{\ell -1}} \}$, we have
\begin{align}
\|f_\ell-g_\ell\|_\infty
&\le \omega_{t_\ell,\infty} \Big( \| \phi_{\ell-1} \circ f_{\ell-1} - \phi_{\ell-1} \circ g_{\ell-1} \|_\infty \nonumber\\
&\qquad \qquad \qquad + \sup_{x\in [0,1]^{d_x}} \max_{i \in [d_{\ell-1}]} \left\{ \sigma( g_{\ell -1}(x)_i ) - \varphi_n( g_{\ell -1}(x)_i )  \right\} \Big)\nonumber\\
& \le  \omega_{t_\ell,\infty} \left( \omega_{\phi_{\ell-1},\infty} (\| f_{\ell-1} - g_{\ell-1} \|_\infty) +  \delta \right), \label{ineq:sup-norm}
\end{align}
with $\| f_1 - g_1 \|_p = \| t_1 - t_1 \|_p = 0.$
Again, note that $\omega_{\phi_{\ell-1},\infty}$ is well-defined on $[0,1]^{d_x}$ since $\varphi_{\ell-1}$ is uniformly continuous on $[0,1]^{d_x}$.

Consequently, by iteratively applying (\ref{ineq:sup-norm}), we get
\begin{align}
\|f-g\|_\infty & \le  \omega_{t_L,\infty} \left( \omega_{\phi_{L-1},\infty}\left( \| f_{L-1} - g_{L-1} \|_\infty\right) + \delta \right) \nonumber \\
&\le  \omega_{t_L,\infty} \left( \omega_{\phi_{L-1},\infty}\left( \omega_{t_{L-1},\infty} \left( \omega_{\phi_{L-2},\infty} \left( \| f_{L-2} - g_{L-2} \|_\infty \right) + \delta \right) \right) +\delta \right) \nonumber \\
&~~\vdots \nonumber \\
&\le  \omega_{t_L,\infty} \left( \omega_{\phi_{L-1},\infty} \left( \cdots \left( \omega_{t_3,\infty} \left( \omega_{\phi_{2},\infty} ( \| f_2 - g_2 \|_\infty) + \delta \right)\right) \cdots \right)+\delta\right) \nonumber \\ 
&\le  \omega_{t_L,\infty} \left( \omega_{\phi_{L-1},\infty} \left( \cdots \left( \omega_{t_3,\infty} \left( \omega_{\phi_{2},\infty} \left( 
\omega_{t_2,\infty} \left( \delta  \right)
\right) + \delta \right)\right) \cdots \right)+\delta\right) \label{ineq:mono-unif-ub}
\end{align}

Therefore, we can bound the right-hand side (\ref{ineq:mono-unif-ub}) of the above inequality within arbitrary $\varepsilon >0 $, by choosing sufficiently small $\delta > 0$.
Hence, it completes the proof of \cref{lem:inj-approx}.

\section{Minimum width for $L^p$ approximation of RNNs}\label{sec:rnn}
In \cref{thm:lp-ub,cor:general-lp}, we prove the exact minimum width for $L^p$ approximation via networks using  $\relu$ or $\relul$ activation functions.
Using similar proof techniques, we investigate the minimum width for $L^p$ approximation for other network architectures: recurrent neural networks (RNNs) and bidirectional RNNs (BRNNs) in this section.

\subsection{Additional notations}

We first introduce additional notations that will be used throughout this section.
Given a length $T$ sequence of $d$-dimensional vectors $x \in \mathbb R^{d \times T}$ (i.e., $x$ is a matrix of $d$ rows and $T$ columns), we denote a token at index $t\in[T]$ by $x[t] \in \mathbb R^d$ (i.e., the {$t$}-th column of $x$) and tokens from index $t_1\in[T]$ to $t_2 \in [T]$ by $x[t_1 : t_2] \in \mathbb R^{d \times {(t_2-t_1+1)}}$ for $t_1 < t_2$ (i.e., the submatrix of $x$ consisting of its {$t_1$}-th,\dots,{$t_2$}-th columns). We define recurrent cells used in RNN and BRNN architectures as follows.

\begin{itemize}[leftmargin=15pt]
\item {\bf RNN cell.} A recurrent cell $\vec{R}_\ell$ of the layer $\ell$ with hidden dimension $d_\ell$ maps an input sequence $x = (x[1],\dots,x[T]) \in \mathbb R^{d_\ell \times T}$ to an output sequence $y = (y[1],\dots,y[T]) \in \mathbb R^{d_\ell \times T}$ such that
\begin{align*}
    y[t+1] = \vec{R}_\ell(x)[t+1] \defeq \phi_\ell(\vec{W}_{\ell,1} \vec{R}_\ell(x)[t] + \vec{W}_{\ell,2} x[t+1] + \vec{b}_\ell),
\end{align*}
where $\phi_\ell(x_1,\dots,x_{d_\ell}) = (\sigma(x_1),\dots,\sigma(x_{d_\ell}))$ is a coordinate-wise activation function, and $\vec{W}_{\ell,1}, \vec{W}_{\ell,2} \in \mathbb R^{d_\ell \times d_\ell}$ and $\vec{b}_\ell \in \mathbb R^{d_\ell}$ are the weight parameters. The initial hidden state $\vec{R}_\ell(x)[0]$ is set to be $0\in \mathbb R^{d_\ell}$.

\item {\bf BRNN cell.} %
A bidirectional recurrent cell $\vecev{R}_\ell$ of the layer $\ell$ with hidden dimension $d_\ell$ consists of a pair of recurrent cells $\vec{R}_\ell$, $\cev{ R}_\ell$ with the same hidden dimension, and additional weight parameters $A_\ell,B_\ell \in \mathbb R^{d_\ell\times d_\ell}$ such that
\begin{align*}
    \vec{R}_\ell(x)[t+1] &= \phi_\ell(\vec{W}_{\ell,1} \vec{R}_\ell(x)[t] + \vec{W}_{\ell,2} x[t+1] + \vec{b}_\ell), \\
    \cev{ R}_\ell(x)[t-1] &\defeq \phi_\ell(\cev{W}_{\ell,1} \cev{R}_\ell(x)[t] + \cev{W}_{\ell,2} x[t-1] + \cev{b}_\ell),\\
    y[t+1] = \vecev{R}_\ell(x)[t] &\defeq A_\ell  \vec{R}_\ell(x)[t] + B_\ell \cev{R}_\ell(x)[t],
\end{align*}
where the initial hidden states $\vec{R}_\ell(x)[0]$ and $\cev R_\ell(x)[T+1]$ are set to be $0 \in \mathbb R^{d_\ell}$.

\item {\bf Network architecture.} Given an activation function $\sigma : \mathbb R \to \mathbb R$, %
token-wise linear maps $P: \mathbb R^{d_x \times T} \to \mathbb R^{d \times T}$ and $Q : \mathbb R^{d \times T} \to \mathbb R^{d_y \times T}$ (i.e., there are some linear maps $\phi:\mathbb R^{d_x}\to\mathbb R^d$ and $\psi:\mathbb R^{d}\to\mathbb R^{d_y}$ such that  $P(x)[t]=\phi(x[t])$ and $Q(x)[t]=\psi(x[t])$),
and $L$ recurrent cells $ \vec{R}_1, \dots, \vec{R}_L$ with hidden dimensions $d_1, \dots, d_L$, we define an RNN $f$ as follows:
\begin{align*}
    f \defeq  Q \circ  \vec{R}_L \circ \cdots \circ \vec{R}_1 \circ P.
\end{align*}
\end{itemize}
We denote a neural network $f$ with an activation function $\sigma$ by a ``$\sigma$ RNN''.
If we replace RNN cells $\vec{R}_1, \dots, \vec{R}_L$ to BRNN cells $\vecev{R}_1, \dots,  \vecev{R}_L$, then we denote a function $f$ by a ``$\sigma$ BRNN''.
We define the width of RNN (or BRNN) $f$ as the maximum over $d_1, \dots, d_L$.

We now introduce function spaces to universally approximate via RNNs and BRNNs.
Given $T \in \mathbb N$, we define the target function class $L^p(\mathcal X^{T}, \mathcal Y^{T} )$, which consists of all $L^p$ sequence-to-sequence functions with length $T$ from $\mathcal X \subset \mathbb R^{d_x}$ to $\mathcal Y \subset \mathbb R^{d_y}$, endowed with the entry-wise $L^p$-norm: $\| f \|_{p,p} \defeq(\int_{\mathcal{X}^T} \|f(x)\|_{p,p}^p dx)^{1/p}$ where $\| \cdot \|_{p,p}$ is the $L_{p,p}$ norm, i.e., an entry-wise matrix norm.
Unlike BRNNs, output tokens of RNNs at index $t\in[T]$ only depend on $x[1:t] \in \mathbb R^{d_x \times t}$.
We refer to such functions that only depend on past information as {\it the past-dependent functions}. 
Namely, a function $f:\mathbb R^{d_1\times T}\to\mathbb R^{d_2\times T}$ is past-dependent if 
$$f(x)[t]=g_t(x[1:t])$$
for some $g_t:\mathbb R^{d_1\times t}\to\mathbb R^{d_2}$ for all $t\in[T]$.
For a target function class for universal approximation using RNNs, we consider past-dependent $L^p([0,1]^{d_x\times T},\mathbb R^{d_y\times T})$, which is a space of all past-dependent functions $f$ such that $f\in L^p([0,1]^{d_x\times T},\mathbb R^{d_y\times T})$.
For a target function class for universal approximation using BRNNs, we consider $L^p([0,1]^{d_x\times T},\mathbb R^{d_y\times T})$.

Before describing our results, we introduce a  recent work for universal approximation of RNNs \citep{song23}.
\citet{song23} show that the upper bound on the minimum width for universal approximation is independent of the length of the input sequences. In particular, they consider unbounded domain and prove that width $\max\{d_x + 1, d_y\}$ is necessary and sufficient for $\relu$ RNNs to be dense in the past-dependent $L^p(\mathbb R^{d_x\times T},\mathbb R^{d_y \times T})$ and the same width $\max\{d_x + 1, d_y\}$ is sufficient for $\relu$ BRNNs to universally approximate $L^p(\mathbb R^{d_x\times T},\mathbb R^{d_y\times T})$.

\begin{table}[t]
\small
\begin{center}
\caption{A known bounds on the minimum width for $L^p$ approximation via RNNs and BRNNs using $\relu$ or $\relul$ activation functions. In this table, $p \in [1, \infty)$ and all results with the domain $[0,1]^{d_x\times T}$ extends to $\mathcal K^T$ where $\mathcal K$ denotes an arbitrary compact set in $\mathbb R^{d_x}$.}\label{table:summary2}
\begin{tabular}{|c| c | c  c | c |} 
 \hline
 {\bf Reference} & {\bf Network} & {\bf Function class} & {\bf Activation} $\sigma$ & {\bf Upper\,/\,lower bounds} \\ 
 \hline\hline
 \multirow{2}{*}{\cite{song23}} & {RNN} & $L^p(\mathbb R^{d_x\times T},\mathbb R^{d_y\times T})^\mathparagraph$ & $\relu$ & $w_{\min} = \max\{d_x+1,d_y\}$\\
                                & {BRNN} & $L^p(\mathbb R^{d_x\times T},\mathbb R^{d_y \times T})$ & $\relu$ & $w_{\min} \le \max\{d_x+1,d_y\}$ \\
 \hline
 \hline
 \rowcolor{gray!30} 
{\bf \cref{thm:rnn_lp}} &  & & $\relu$ & $ w_{\min} = \max\{d_x, d_y, 2\}$ \\
 \rowcolor{gray!30} 
{\bf \cref{thm:rnn_lp-relulike}} & \multirow{-2}{*}{RNN}
 &  \multirow{-2}{*}{$L^p([0,1]^{d_x\times T}, \mathbb R^{d_y\times T})^\mathparagraph$} & $\relul^{\mathsection}$ & $ w_{\min} = \max\{d_x, d_y, 2\}$ \\
 \hline
\rowcolor{gray!30} 
 &  &  & $\relu$ & $ w_{\min} \le \max\{d_x, d_y, 2\}$ \\
 \rowcolor{gray!30} 
 \multirow{-2}{*}{\bf \cref{thm:brnn_lp}} & \multirow{-2}{*}{BRNN}
 &  \multirow{-2}{*}{$L^p([0,1]^{d_x\times T}, \mathbb R^{d_y\times T})$} & $\relul^{ \|\ }$ & $ w_{\min} \le \max\{d_x, d_y, 2\}$ \\
 \hline
\end{tabular}
\end{center}

$\mathparagraph$ requires the class to consist of past-dependent functions.\\
$\mathsection$ includes $\softplus$, Leaky-$\relu$, $\elu$, $\celu$, $\selu$, $\gelu$, $\silu$, and $\mish$ where $\gelu$, $\silu$, and $\mish$ requires $d_x+d_y\ge3$.\\
$ \|\ $ do not require $d_x+d_y\ge3$ for $\gelu$, $\silu$, and $\mish$.
\end{table}

\subsection{Our results}
We are now ready to introduce our results on a compact domain. The first result characterizes the exact minimum width of RNNs to be dense in past-dependent $L^p([0,1]^{d_x \times T},\mathbb R^{d_y \times T})$.
The proof of \cref{thm:rnn_lp,thm:rnn_lp-relulike} are presented in \cref{pfsec:rnn,sec:pfthm:rnn_lp-relulike}, respectively.

\begin{theorem}\label{thm:rnn_lp}
For any $T \in \mathbb N$, $w_{\min} = \{ d_x,d_y,2 \}$ for $\relu$ RNNs to be dense in past-dependent $L^p([0,1]^{d_x \times T},\mathbb R^{d_y \times T})$.
\end{theorem}
\begin{theorem}\label{thm:rnn_lp-relulike}
For any $T \in \mathbb N$, $w_{\min} = \{ d_x,d_y,2 \}$ for $\varphi$ RNNs to be dense in past-dependent $L^p([0,1]^{d_x \times T},\mathbb R^{d_y \times T})$ if $\varphi \in \{\elu, \text{\rm Leaky-}\relu, \softplus, \celu, \selu \}$, or $\varphi \in \{\gelu, \silu, \mish \}$ and $d_x + d_y \ge 3$.
\end{theorem}

\cref{thm:rnn_lp,thm:rnn_lp-relulike} characterize the minimum width of RNNs using $\relu$ or $\relul$ activation functions to be dense in past-dependent $L^p([0,1]^{d_x \times T},\mathbb R^{d_y \times T})$ is exactly $\max\{d_x,d_y,2\}$, which coincides with the fully-connected network case (\cref{thm:lp-ub,cor:general-lp}).
Further, \cref{thm:rnn_lp} shows a dichotomy between the minimum width of $\relu$ RNNs for $L^p$ approximation on the compact domain and the whole Euclidean space. A similar observation also holds for RNNs using $\relul$ activation functions using \cref{thm:rnn_lp-relulike}.

In order to prove the upper bound $w_{\min} \le \{ d_x,d_y,2 \}$ in \cref{thm:rnn_lp,thm:rnn_lp-relulike}, we use coding-based proof techniques as in \citep{song23} but with different coding schemes (e.g., as in \cref{lem:encoder}).
The lower bound $w_{\min} \ge \{d_x, d_y, 2 \}$ in \cref{thm:rnn_lp} directly follows from the facts that for any $\varphi \in \{ \relu \} \cup \relul$ and $\varphi$ RNN $f$, $f(x)[1] = Q \circ \vec{R}_L \circ \cdots \circ \vec{R}_1 \circ P(x)[1]$ is a $\varphi$ network and $w_{\min} \ge \max\{d_x, d_y, 2\}$ is necessary for $\varphi$ networks to be dense in $L^p([0,1]^{d_x}, \mathbb R^{d_y})$ (\cref{thm:lp-ub,cor:general-lp}).

Our next result shows that the same upper bound in \cref{thm:rnn_lp} also holds for $\relu$ BRNNs and BRNNs using $\relu$ or $\relul$ activation functions. The proof of \cref{thm:brnn_lp} is presented in \cref{pfsec:brnn}.

\begin{theorem}\label{thm:brnn_lp}
For any $T \in \mathbb N$ and $\varphi \in \{\relu\} \cup \relul$, $w_{\min} \le \{ d_x,d_y,2 \}$ for $\varphi$ BRNNs to be dense in $L^p([0,1]^{d_x\times T}, \mathbb R^{d_y \times T})$.

\end{theorem}

\subsection{Proof of \cref{thm:rnn_lp}}\label{pfsec:rnn}
\subsubsection{Proof outline for $\relu$ RNNs}
In this section, we show that for any past-dependent $f^* \in L^p([0,1]^{d_x \times T}, \mathbb R^{d_y \times T})$ and $\varepsilon > 0$, there exists a $\relu$ RNN $f : [0,1]^{d_x \times T} \to \mathbb R^{d_y \times T}$ of width $\max\{d_x, d_y, 2\}$ such that
\begin{align*}
     \| f - f^* \|_{p,p} \le \varepsilon.
\end{align*}
Without loss of generality, we restrict the codomain to $[0,1]^{d_y \times T}$.
Then, since continuous functions in $C([0,1]^{d_x \times T},\mathbb R^{d_y \times T})$ are dense in $L^p([0,1]^{d_x \times T},\mathbb R^{d_y \times T})$ \citep{rudin}, it suffices to prove the following statement:
for any $\varepsilon>0$, $f' \in C([0,1]^{d_x \times T}, [0,1]^{d_y \times T})$, there exists a $\relu$ RNN $f : [0,1]^{d_x \times T} \to [0,1]^{d_y \times T}$ of width $\max\{d_x, d_y, 2\}$ satisfying
\begin{align*}
    \| f' - f \|_{p,p} \le {\varepsilon}.
\end{align*}
We explicitly construct such $\relu$ RNN $f$ using the coding scheme. To describe this, we present the following lemmas where the proofs of \cref{lem:tokenwise-encoder,lem:rnn-encoder,lem:rnn-decoder} are presented in \cref{pfsec:lem:tokenwise-encoder,pfsec:lem:rnn-encoder,pfsec:lem:rnn-decoder}, respectively.

\begin{lemma}\label{lem:tokenwise-encoder}
Given $T\in\mathbb N$ and $\alpha,\beta>0$, there exist disjoint measurable sets $\mathcal T_1,\dots,\mathcal T_k\subset[0,1]^{d_x}$ and
a $\relu$ RNN $g^{\dagger}:[0,1]^{d_x\times T}\to\mathbb R^T$ of width $\max\{d_x,2\}$ such that 
\begin{itemize}[leftmargin=15pt]
    \item $\diam(\mathcal T_i)\le\alpha$ for all $i\in[k]$,
    \item $\mu_{d_x}\big(\bigcup_{i=1}^k\mathcal T_i\big)\ge1-\beta$, and
    \item if $x\in[0,1]^{d_x\times T}$ satisfies $x[t]\in\mathcal T_{i_t}$ for all $t\in[T]$, then $g^{\dagger}(x)[t]=c_{i_t}$ for all $t\in[T]$, for some distinct $c_1,\dots,c_k\in\mathbb R$.
\end{itemize}
\end{lemma}
\cref{lem:tokenwise-encoder} states that for any $\alpha,\beta>0$, there exist $\mathcal T_1,\dots,\mathcal T_k$ satisfying properties in \cref{lem:tokenwise-encoder} and a $\relu$ RNN $g^\dagger$ of width $\max\{d_x,2\}$ that assigns distinct codewords to each token $x[t]\in\mathcal T_{i_t}$ for $t\in[T]$.
However, unlike the fully-connected network case, the $t$-th token of the RNN output must be a function of $x[1:t]$. To encode information of $x[1:t]$, we introduce the following lemma.
\begin{lemma}\label{lem:rnn-encoder}
Given $T\in\mathbb N$ and distinct $c_1,\dots,c_k\in\mathbb R$, there exist 
\begin{itemize}[leftmargin=15pt]
\item distinct $a_{j}\in\mathbb R$ for all $j\in[k]^t$, and
\item a $\relu$ RNN $g^{\ddagger}:\mathbb R^{1\times T}\to\mathbb R^{1\times T}$ of width $2$ such that for any $(c_{i_1},\dots,c_{i_T})$ with $i_1,\dots,i_T\in[k]$
\end{itemize}
$$g^{\ddagger}(x)[t]=a_{j_t}$$
for all $t\in[T]$ where $j_t=(i_1,\dots,i_t)$.
\end{lemma}

\cref{lem:rnn-encoder} states that for any set of codewords $\{c_1,\dots,c_k\}$, %
there exists an RNN encoder implemented by a $\relu$ RNN $g^\ddagger$ of width $2$ that maps any vector of codewords $(c_{i_1},\dots,c_{i_t})$ of length $t\in[T]$ into a single scalar codeword $a_{(i_1,\dots,i_t)}$ with the following property: different vectors are mapped to distinct scalar codewords.

By \cref{lem:tokenwise-encoder,lem:rnn-encoder}, one can observe that for any $\alpha,\beta>0$, there exist $\mathcal T_1,\dots,\mathcal T_k\subset[0,1]^{d_x}$, a $\relu$ RNN $g'$ of width $\max\{d_x,2\}$, and $a_j\in\mathbb R$ for all $j\in\bigcup_{t=1}^T[k]^t$ satisfying the following properties:
\begin{itemize}[leftmargin=15pt]
    \item $\diam(\mathcal T_i)\le\alpha$ for all $i\in[k]$,
    \item $\mu_{d_x}\big(\bigcup_{i=1}^k\mathcal T_i\big)\ge1-\beta$, 
    \item $a_j\ne a_{j'}$ if $j\ne j'$, and
    \item if $x\in[0,1]^{d_x\times T}$ satisfies $x[t]\in\mathcal T_{i_t}$ for all $t\in[T]$, then $g'(x)[t]=a_{(i_1,\dots,i_t)}$.
\end{itemize}
Namely, if $x\in\mathcal T_{i_1}\times\cdots\times\mathcal \mathcal T_{i_T}$, then $g'(x)[t]=a_{(i_1,\dots,i_t)}$.

We next construct a $\relu$ RNN that maps each $(a_{(i_1)},\dots,a_{(i_1,\dots,i_T)})$ to some $y\in\mathbb R^{d_y\times T}$ such that $y[t]$ approximates $f'(\mathcal T_{i_1}\times\cdots\times\mathcal T_{i_T})[t]$ for all $t\in[T]$.
\begin{lemma}\label{lem:rnn-decoder}
Given $T\in\mathbb N$, $p\ge1$, $\gamma>0$, distinct $a_1,\dots,a_m\in\mathbb R$, and $v_1,\dots,v_m\in\mathbb R^{d_y}$, there exists a $\relu$ RNN $h:\mathbb R^{1\times T}\to[0,1]^{d_y \times T}$ of width $\max\{d_y,2\}$ such that for any $x=(a_{j_1},\dots,a_{j_T})$ with $j_1,\dots,j_T\in[m]$,
$$\|h(x)[t]-v_{j_t}\|_p\le\gamma$$
for all $t\in[T]$.
\end{lemma}
By combining \cref{lem:tokenwise-encoder,lem:rnn-encoder,lem:rnn-decoder}, one can observe that for any $\alpha,\beta,\gamma>0$, there exist $\mathcal T_1,\dots,\mathcal T_k\subset[0,1]^{d_x}$, a $\relu$ RNN $f$ of width $\max\{d_x,d_y,2\}$, and $v_j\in\mathbb R^{d_y}$ for all $j\in\bigcup_{t=1}^T[k]^t$ satisfying the following properties:
\begin{itemize}[leftmargin=15pt]
    \item $\diam(\mathcal T_i)\le\alpha$ for all $i\in[k]$,
    \item $\mu_{d_x}\big(\bigcup_{i=1}^k\mathcal T_i\big)\ge1-\beta$, 
    \item if $x\in[0,1]^{d_x\times T}$ satisfies $x[t]\in\mathcal T_{i_t}$ for all $t\in[T]$, then $$\|f(x)[t]-f'(x)[t]\|_p\le\omega_{(p,p),F,f'}\left(\alpha\sqrt{T}\right) + (d_y T^{2} \beta)^{1/p} + T^{1/p}\gamma$$ for all $t\in[T]$,\footnote{$\omega_{(p,p),F,f'}$ denotes the modulus of continuity of $f'$ in the $L_{p,p}$-norm and Frobenius-norm: $\|f'(x)-f'(x')\|_{p,p}\le\omega_{(p,p),F,f'}(\|x-x'\|_F)$ for all $x,x'\in[0,1]^{d_x \times T}$.} and 
    \item if $x\in[0,1]^{d_x\times T}$ satisfies $x[t]\notin\mathcal T_{i_t}$ for some $t\in[T]$, then $f(x)\in[0,1]^{d_y\times T}$.
\end{itemize}
We will show that such RNN $f$ of width $\max\{d_x,d_y,2\}$ satisfies $\|f-f'\|_{p,p}\le\varepsilon$ under proper choices of $\alpha,\beta,\gamma>0$.

\subsubsection{Our choices of $\alpha,\beta,\gamma$ for $\relu$ RNNs}\label{sec:rnn-error-estimate}

We choose sufficiently small $\alpha > 0$ so that $\omega_{(p,p),F,f'}(\alpha\sqrt{T}) \le \varepsilon/{2^{1+1/p}}$, $\beta = \varepsilon^p/(2 d_y T^2)$, and $\gamma = \varepsilon/(2^{1+1/p}T^{1/p})$.
For convenience, we use $\mathcal T\defeq\bigcup_{i=1}^k\mathcal T_i$.
Under this setup, we bound the error using the following inequality:
\begin{align}
& \| f - f' \|_{p,p}^p =  \int_{[0,1]^{d_x \times T}} \| f'(x) - f(x)  \|_{p,p}^p d\mu_{dxT}\notag\\
&\le T \times \sup_{1 \le t \le T} \int_{[0,1]^{d_x \times T} \setminus \mathcal T^T} \| f'(x)[t] - f(x)[t]  \|_p^p d\mu_{dxT}+  \int_{\mathcal T^T} \| f'(x) - f(x)  \|_{p,p}^p d\mu_{dxT}.\label{eq:relu-rnn-err}%
\end{align}
We first bound the first term in RHS of \cref{eq:relu-rnn-err}. Note that both $ f $ and $ f' $ have codomain is $ [0,1]^{dy} $.
\begin{align}
	&T \times \sup_{1 \le t \le T} \int_{[0,1]^{d_x \times T} \setminus \mathcal T^T} \| f'(x) - f(x)  \|_{p,p}^p d\mu_{dxT}\notag\\
	&= T\times d_y  \mu_{d_xT} \left( \bigcup_{j=1}^T  [0,1]^{d_x \times (T-j)} \times ( [0,1]^{d_x} \setminus \mathcal T) \times [0,1]^{d_x \times (j-1)}  \right)\notag \\
	&  \le Td_y  \sum _{j=1}^{T}( 1- \mu_{d_x} ( \mathcal T  ) )
	\le d_y T^2 \beta \le \varepsilon^p/2.\label{eq:relu-rnn-err2}
\end{align}
We next bound the second term in RHS of \cref{eq:relu-rnn-err} using Minkowski's inequality as follows:
\begin{align}
&\int_{\mathcal T^T} \| f'(x) - f(x)  \|_{p,p}^p d\mu_{dxT}\notag\\
& =\sum_{i_1,\dots,i_T \in [k]}\int_{ {\prod_{s=1}^T} \mathcal T_{i_s} }\|f(x)-f'(x)\|_{p,p}^p d\mu_{dxT}\notag\\
& \le  \sum_{i_1,\dots,i_T \in [k]} \int_{\prod_{s=1}^T \mathcal T_{i_s}} ( \|f'(z_{j_T})-f'(x)\|_{p,p} + \|f(x)-f'(z_{j_T})\|_{p,p})^p d\mu_{dxT}\notag\\
&  \le \sum_{i_1,\dots,i_T \in [k]} \left[\left(\int_{\prod_{s=1}^T \mathcal T_{i_s}}  \|f'(z_{j_T})-f'(x)\|_{p,p}^{p}d\mu_{dxT}\right)^{1/p}\right.\notag\\ 
&\left. \qquad \qquad \qquad \qquad \qquad \qquad \qquad
+ \left(\int_{\prod_{s=1}^T \mathcal T_{i_s}}\|f(x)-f'(z_{j_T})\|_{p,p}^p d\mu_{dxT}\right )^{1/p}\right ]^{p}\notag\\
&  \le \left [\left(\sum_{i_1,\dots,i_T \in [k]} \int_{\prod_{s=1}^T \mathcal T_{i_s}}  \|f'(z_{j_T})-f'(x)\|_{p,p}^{p}d\mu_{dxT}\right)^{1/p} \right.\notag\\
&\left. \qquad \qquad \qquad \qquad \qquad \qquad \qquad
+ \left(\sum_{i_1,\dots,i_T \in [k]}\int_{\prod_{s=1}^T \mathcal T_{i_s}}\|f(x)-f'(z_{j_T})\|_{p,p}^p d\mu_{dxT}\right )^{1/p}\right]^{p} \notag\\
&  \le \left[\left(\sum_{i_1,\dots,i_T \in [k]} \int_{\prod_{s=1}^T \mathcal T_{i_s}}  \left(\omega_{(p,p),F,f'} \left(\|z_{j_T}-x\|_F\right)\right)^{p}d\mu_{dxT}\right)^{1/p} \right.\notag\\
& \left. \qquad \qquad \qquad \qquad \qquad \qquad \qquad
+ \left(\sum_{i_1,\dots,i_T \in [k]}\int_{\prod_{s=1}^T \mathcal T_{i_s}}\sum_{t \in [T]}\|f(x)[t]-v_{j_t}\|_{p}^p d\mu_{dxT}\right )^{1/p}\right]^{p} \notag\\
&  \le \left [\left(\sum_{i_1,\dots,i_T \in [k]} \int_{\prod_{s=1}^T \mathcal T_{i_s}}  \left(\omega_{(p,p),F,f'} \left(\alpha\sqrt{T}\right)\right)^{p}d\mu_{dxT}\right)^{1/p} \right.\notag\\
& \left. \qquad \qquad \qquad \qquad \qquad \qquad \qquad
+ \left(\sum_{i_1,\dots,i_T \in [k]}\int_{\prod_{s=1}^T \mathcal T_{i_s}}T\gamma^{p} d\mu_{dxT}\right )^{1/p}\right]^{p} \notag\\
&  \le  \left(\omega_{(p,p),F,f'}\left(\alpha\sqrt{T}\right) + T^{1/p}\gamma\right)^p \le \varepsilon^p/2\label{eq:relu-rnn-err3}
\end{align}
where $z_{j_T} \in \prod_{s=1}^T \mathcal T_{i_s}$ for all $i_s \in [k]$.
The second term in the above bound used our construction of $ v_{j_t} $ and $ f = h\circ g^{\ddagger}\circ g^{\dagger} $.
By combining \cref{eq:relu-rnn-err,eq:relu-rnn-err2,eq:relu-rnn-err3}, we have
$$\|f-f^*\|_{p,p}\le\frac\varepsilon2+\frac\varepsilon2\le\varepsilon.$$
This completes the proof of \cref{thm:rnn_lp}.

\subsection{Proof of \cref{thm:rnn_lp-relulike}}\label{sec:pfthm:rnn_lp-relulike} 
In this section, we prove that any $\relu$ RNN (or BRNN) $f$ can be approximated by an RNN (or BRNN) $g$ of the same width using any of $\relul$ activation functions, within any uniform error. Note that a $\relu$ RNN $f$ of width $w$ does not imply $f$ is a $\relu$ network with width $w$. Hence, we need the following extended definition for analysis of $\relu$ RNN.

Given an activation function $\sigma:\mathbb R\to\mathbb R$, we define a {\it $\sigma $ token-network} as follows:
\begin{align}\label{eq:def-rnn}
    f(x_1,x_2,\dots,x_T) \defeq \psi_L \circ \psi_{L-1} \circ \cdots \circ \psi_2 \circ \psi_1,
\end{align}
where $\psi_\ell$ is one of the following operations.

\begin{itemize}[leftmargin=15pt]
    \item applying affine transformation $t(\cdot)$ on $k$-th token: \\
    $$\psi_{t(\cdot),k}(x_1,x_2,\dots, x_T)\defeq(x_1,x_2,\dots,t(x_k),\dots, x_T)$$
    where $W_t \in \mathbb R^{d_t \times d}$, $b_t \in \mathbb R^{d_t}$, $x_k \in \mathbb R^{d}$ and $t(x)=W_tx+b_t$ for some $d, d_t\in \mathbb N$. 
    \item element-wise $\sigma$ activation on $k$-th token: 
    $$\psi_{\sigma,k}(x_1,x_2,\dots, x_T)\defeq(x_1,x_2,\dots,\phi_\sigma(x_k),\dots, x_T)$$
    where $\phi_\sigma$ is an element-wise activation function.
    \item copying the $k$-th token $\psi_c$ to a new token:
    $$\psi_{c,k}(x_1,x_2,\dots, x_T)\defeq(x_1,x_2,\dots, x_T,x_k)$$
    \item adding two tokens with the same dimension into a new token:
        $$\psi_{s,k,l}(x_1,x_2,\dots, x_T)\defeq(x_1,x_2,\dots, x_T,x_k+x_l)$$
    \item deleting $k$-th token:
    $$\psi_{d,k}(x_1,x_2,\dots, x_T)\defeq(x_1,x_2,\dots x_{k-1}, x_{k+1},\dots, x_T)$$
\end{itemize}
The width $w$ of a $\sigma$ token-network $f$ is defined as the maximum of input/output dimensions of affine transformations $t$ that are applied in $f$. 
Remark that $\sigma $ RNN (or BRNN) of width $w$ is a $\sigma$ token-network with width $w$. 
If we define $\|(x_1,\dots,x_T)\|_{\infty}\defeq \sup_{t\in[T]} \|x_t\|_\infty$, we can apply the same method as in the proof of \cref{lem:inj-approx} in \cref{sec:pflem:inj-approx}.
When $\psi$ is either copying or deleting, then 
$$\|\psi(X)-\psi(Y)\|_{\infty} \le \|X-Y\|_{\infty}$$
and when $\psi$ is adding tokens, then 
$$\|\psi(X)-\psi(Y)\|_{\infty} \le 2\|X-Y\|_{\infty}$$
above error bound holds.

So for a given $\relu$ token-network $f$, the following inequality holds for every $\psi$ in $f$:
$$\|\psi(X)-\psi(Y)\|_{\infty} \le \max\{2,M\}\|X-Y\|_{\infty}$$
where $M$ is maximum value of norm of affine transformation $\|W_t\|_\infty$ in $f$.
Therefore, using the identical method as in \cref{sec:pflem:inj-approx}, we are able to construct $\relul$ token-network $g$ such that 
$$\|f(X)-g(X)\|_\infty\le \varepsilon.$$
Since uniform convergence of functions in compact domain implies $p$-norm convergence, we are able to extend the result of $\relu$ RNNs to RNNs using $\relul$ activation functions, hence the proof of \cref{thm:rnn_lp-relulike} is completed.

\subsection{Proof of \cref{thm:brnn_lp}}\label{pfsec:brnn}
In this proof, we follow the discussion in \cref{pfsec:rnn}. 
We use \cref{lem:tokenwise-encoder,lem:rnn-decoder} and a modified version of \cref{lem:rnn-encoder} to construct encoder and decoder BRNNs.
The modified lemma is as follows:
\begin{lemma}\label{lem:brnn-encoder}
Given $T\in\mathbb N$ and distinct $c_1,\dots,c_k\in\mathbb R$, there exist 
\begin{itemize}[leftmargin=15pt]
\item distinct $a_{j_t,\bar{j}_t}\in\mathbb R$ for all $t\in[T]$, $j_t \in[k]^t$, and $\bar{j}_t \in[k]^{T-t+1}$, and
\item a $\relu$ RNN $g^{\ddagger}:\mathbb R^{1\times T}\to\mathbb R^{1\times T}$ of width $2$ 
such that for any $x:=(c_{i_1},\dots,c_{i_T})$ with $i_1,\dots,i_T\in[k]$
\end{itemize}
$$g^{\ddagger}(x)[t]=a_{j_t,\bar{j}_t}$$
for all $t\in[T]$ where $j_t=(i_1,\dots,i_t)$ and $\bar{j}_t = (i_t,\dots,i_T)$.
\end{lemma}
Note that $a_{j_t}$ in \cref{lem:rnn-encoder} only depends on the past whereas $a_{j_t,\bar{j}_t}$ in \cref{lem:brnn-encoder} depends both on the past and future. 
The proof of \cref{lem:brnn-encoder} is provided in \cref{pfsec:lem:brnn-encoder}.

Now, we are ready to construct our BRNN model $f$.
First, \cref{lem:tokenwise-encoder} ensures that there exist
$\mathcal T_1,\dots,\mathcal T_k\subset[0,1]^{d_x}$ and a $\relu$ RNN $g^\dagger$ of width $\max\{d_x,2\}$. 
Also from \cref{lem:brnn-encoder}, there exist $a_{j,\bar{j}} \in\mathbb R$ for all $j, \bar{j}\in\bigcup_{t=1}^T[k]^t\times[k]^{T-t+1}, j[t]=\bar{j}[0]$, and a $\relu$ RNN $g^\ddagger$ satisfying the following properties:

\begin{itemize}[leftmargin=15pt]
    \item $\diam(\mathcal T_i)\le\alpha$ for all $i\in[k]$,
    \item $\mu_{d_x}\big(\bigcup_{i=1}^k\mathcal T_i\big)\ge1-\beta$, 
    \item $a_{j,\bar{j}}\ne a_{j',\bar{j'}}$ if $(j, \bar{j})\ne (j',\bar{j'})$, and
    \item if $x\in[0,1]^{d_x\times T}$ satisfies $x[t]\in\mathcal T_{i_t}$ for all $t\in[T]$, then $g^\ddagger\circ g^\dagger(x)[t]=a_{(i_1,\dots,i_t),(i_t, \dots, i_T)}$.
\end{itemize}
Namely, if $x\in\mathcal T_{i_1}\times\cdots\times\mathcal \mathcal T_{i_T}$, then $g^\ddagger \circ g^\dagger(x)[t]=a_{(i_1,\dots,i_t),(i_t,\dots,i_T)}$.
Now, for a given target function $f^* \in L^p([0,1]^{d_x\times T},\mathbb{R}^{d_y\times T})$, we choose $z_{j_T}\in \prod_{s=1}^T \mathcal{T}_{i_s}$ as in \cref{pfsec:rnn}.
Then, we define:
$$v_{j_t,\bar{j}_t} = f'(z_{j_T})[t].$$
By \cref{lem:rnn-decoder}, one can construct decoder $h$ with respect to $v_{j_t,\bar{j}_t}$ such that 

$$\|h(a_{j_t,\bar{j}_t}) - v_{j_t,\bar{j}_t}\|_p\le\gamma.$$

Note that $h$ is a token-wise function that can be constructed by a $\relu$ BRNN.
Then, the error bound in \cref{pfsec:rnn} indicates that for any $\varepsilon > 0$, we have
$$\|f^* - f\|_{p,p}\le\epsilon$$
where $f = h\circ g^\ddagger \circ g^\dagger$.

Hence, using the extension of $\relu$ token-network to $\relul$ token-network as in \cref{sec:pfthm:rnn_lp-relulike} completes the statement of \cref{thm:brnn_lp}.

\subsection{Proof of \cref{lem:tokenwise-encoder}}\label{pfsec:lem:tokenwise-encoder}

To this end, we recall the statement of \cref{lem:encoder}: For any $\alpha,\beta>0$, there exist disjoint measurable sets $\mathcal T_1,\dots,\mathcal T_k\subset[0,1]^{d_x}$ and a $\relu$ network $g:\mathbb R^{d_x}\to\mathbb R$ of width $\max\{d_x,2\}$ such that 
\begin{itemize}[leftmargin=15pt]
    \item $\diam(\mathcal T_i)\le\alpha$ for all $i\in[k]$,
    \item $\mu_{d_x}\big(\bigcup_{i=1}^k\mathcal T_i\big)\ge1-\beta$, and
    \item $g(\mathcal T_i)=\{c_i\}$ for all $i\in[k]$, for some distinct $c_1,\dots,c_k\in\mathbb R$.
\end{itemize}
Therefore, the statement of \cref{lem:tokenwise-encoder} directly follows from token-wise implementing a $\relu$ network $g$, that is, $g^\dagger(x) = (g(x[1]),\dots,g(x[T]) )$ for any $x \in [0,1]^{d_x \times T}$.

\subsection{Proof of \cref{lem:rnn-encoder}}\label{pfsec:lem:rnn-encoder}

In this section, we explicitly construct a $\relu$ RNN $f : \mathbb R^{1\times T} \to \mathbb R^{1\times T}$ satisfying the statement of \cref{lem:rnn-encoder}.
Without loss of generality, we assume that the distinct points $c_1, \dots, c_k$ are contained in $[0,1]$.
Then, we quantize each distinct point in the binary representation using a token-wise $\relu$ network $g : \mathbb R \to \mathbb R$ of width $2$ such that
for any $K \in \mathbb N$, $\delta>0$, and all $x \in [0,1]\setminus\mathcal D_{K,\delta}$,
\begin{align*}
    g(x) = q_K(x)
\end{align*}
where $\mathcal D_{K,\delta} = \bigcup_{i=1}^{2^K-1}(i \times 2^{-K} - \delta , i \times 2^{-K})$.
The existence of such $g$ is ensured from \cref{lem:mini_decoder}.

Here, we recall the definition of the quantization function.
A quantization function $q_K:[0,1] \to \mathcal C_K$ for $K\in \mathbb N$ and $\mathcal C_K\defeq\{0,2^{-K}, 2\times2^{-K}, 3\times2^{-K},\dots,1-2^{-K}\}$ is defined as
\begin{align*}
    q_K(x) = \max \{c \in \mathcal C_K : c \le x \}.
\end{align*}
One can observe that $g$ preserves the first $K$-bits in the binary representation and discards the rest bits.
Nonetheless, we can ignore the information loss, which is the duplication of points, incurred from the quantization by choosing sufficiently large $K$ and small $\delta$ so that $2^{-(K+1)} < \inf_{i \neq j \in [k]} |c_i-c_j|$ and $\delta < 2^{-(K+2)}$.

Subsequently, we implement a RNN cell $\vec{R} : \mathbb R^{1\times T} \to \mathbb R^{1\times T}$ of width $1$ defined as follows:
\begin{align*}
    \vec{R}(x)[t+1] = \relu(2^{-K} \times \vec{R}(x)[t] + x[t+1]).
\end{align*}
Then, such $\vec{R}$ successfully accumulates $(d_x \times t)$-bits of the binary representation of $x[1:t] \in \mathbb R^{d_x \times t}$ for each $t\in[T]$ since $g (\{c_1,\dots,c_k \}) \subset [0,1]$.

Lastly, let $G : \mathbb R^{1\times T} \to \mathbb R^{1\times T}$ be a $\relu$ RNN of width $2$ such that $G(x) = (g(x[1]),\dots,g(x[T]) )$ for all $x \in \mathbb R^{d_x \times T}$.
Then, the $\relu$ RNN $f = \vec{R}\circ G$ of width $2$ completes the proof of \cref{lem:rnn-encoder}.

\subsection{Proof of \cref{lem:rnn-decoder}}\label{pfsec:lem:rnn-decoder}
From \cref{lem:decoder0}, there exits a $\relu$ network $g:\mathbb R\to[0,1]^{d_y}$ of width $\max\{d_y,2\}$ such that for any $p\ge1$, $\gamma>0$, $m\in\mathbb N$, distinct $a_1,\dots,a_m\in\mathbb R$, and $v_1,\dots,v_m\in\mathbb R^{d_y}$,
\begin{align*}
    \|g(a_i)-v_i\|_p\le\gamma
\end{align*}
for all $i\in[m]$.
Therefore, the statement of \cref{lem:rnn-decoder} directly follows by token-wise implementing such $\relu$ network $g$, that is, $h(x) = (g(a_{j_1}),\dots,g(a_{j_T}))$ for any $j_1,\dots,j_T\in[m]$.
\subsection{Proof of \cref{lem:brnn-encoder}}\label{pfsec:lem:brnn-encoder}

In this section, we follow the similar arguments as in \cref{pfsec:lem:rnn-encoder} to construct  $\relu$ BRNN $f : \mathbb R^{1\times T} \to \mathbb R^{1\times T}$ satisfying the statement of \cref{lem:brnn-encoder}.
Again, we assume that the distinct points $c_1, \dots, c_k$ are contained in $[0,1]$.
Then, there exists a token-wise $\relu$ network $g : \mathbb R \to \mathbb R$ of width $2$ such that
for any $K \in \mathbb N$, $\delta>0$, and all $x \in [0,1]\setminus\mathcal D_{K,\delta}$,
\begin{align*}
    g(x) = q_K(x)
\end{align*}
where $\mathcal D_{K,\delta}$ and quantization function $q_K(x)$ is defined in \cref{pfsec:lem:rnn-encoder}.
Next, choose the same precision $K$ and small enough $\delta$ that satisfies $2^{-(K+1)} < \inf_{i \neq j \in [k]} |c_i-c_j|$ and $\delta < 2^{-(K+2)}$.

We now implement a BRNN cell ${\vecev{R}} : \mathbb R^{1\times T} \to \mathbb R^{1\times T}$ of width $1$ defined as follows:
\begin{align*}
    \vec{R}(x)[t+1] &= \relu(2^{-K} \times \vec R(x)[t] + x[t+1]), \\
    \cev{ R}(x)[t-1] &= \relu(2^{-K} \times \cev R(x)[t] + 2^{-KT}x[t-1]),\\
    {\vecev{R}}(x)[t] &= \vec{R}(x)[t] +\cev{R}(x)[t].
\end{align*}
Then, {$\vecev{R}$} successfully accumulates $(d_x \times t)$-bits for $x[1:t]$ and $(d_x \times (T-t+1))$-bits for $x[t:T]$.
Note that $2^{-KT}x[t-1]$ in $\cev{R}$ enables us to prevent overlapping of information from {$\vecev{R}$} by storing data bits in different positions.

Lastly, let $G: \mathbb R^{1\times T} \to \mathbb R^{1 \times T}$ be a $\relu$ {BRNN} of width $2$ such that $G(x) = (g(x[1]),\dots,g(x[T]) )$ for all $x \in \mathbb R^{d_x \times T}$.
Then, the $\relu$ {BRNN} $f = {\vecev{R}}\circ G$ of width $2$ completes the proof of \cref{lem:brnn-encoder}.

\end{document}